\documentclass{article}
\usepackage{times}
\usepackage{graphicx}
\usepackage[round,authoryear]{natbib}
\usepackage{subfigure}
\usepackage{nicefrac}
\usepackage{booktabs}
\usepackage[T1]{fontenc}
\usepackage{multirow}
\usepackage{amsmath,amsthm,amssymb,amsbsy}
\usepackage{textcomp}
\usepackage{comment}
\usepackage{pdflscape}
\usepackage[hyphens]{url}
\usepackage{xcolor}
\definecolor{mydarkred}{rgb}{0.6,0,0}
\definecolor{mydarkgreen}{rgb}{0,0.6,0}
\usepackage[colorlinks,
linkcolor=mydarkred,
citecolor=mydarkgreen]{hyperref}

\newtheorem{theorem}{Theorem}
\newtheorem{lemma}[theorem]{Lemma}

\ExplSyntaxOn 
\cs_new:Npn \expandableinput #1
  { \use:c { @@input } { \file_full_name:n {#1} } }
\ExplSyntaxOff

\makeatletter 
\providecommand*{\input@path}{}
\g@addto@macro\input@path{{"C:/Users/me/inpath/"}}
\makeatother

\makeatletter
\let\save@mathaccent\mathaccent
\newcommand*\if@single[3]{%
  \setbox0\hbox{${\mathaccent"0362{#1}}^H$}%
  \setbox2\hbox{${\mathaccent"0362{\kern0pt#1}}^H$}%
  \ifdim\ht0=\ht2 #3\else #2\fi
  }
\newcommand*\rel@kern[1]{\kern#1\dimexpr\macc@kerna}
\newcommand*\widebar[1]{\@ifnextchar^{{\wide@bar{#1}{0}}}{\wide@bar{#1}{1}}}
\newcommand*\wide@bar[2]{\if@single{#1}{\wide@bar@{#1}{#2}{1}}{\wide@bar@{#1}{#2}{2}}}
\newcommand*\wide@bar@[3]{%
  \begingroup
  \def\mathaccent##1##2{%
    \let\mathaccent\save@mathaccent
    \if#32 \let\macc@nucleus\first@char \fi
    \setbox\z@\hbox{$\macc@style{\macc@nucleus}_{}$}%
    \setbox\tw@\hbox{$\macc@style{\macc@nucleus}{}_{}$}%
    \dimen@\wd\tw@
    \advance\dimen@-\wd\z@
    \divide\dimen@ 3
    \@tempdima\wd\tw@
    \advance\@tempdima-\scriptspace
    \divide\@tempdima 10
    \advance\dimen@-\@tempdima
    \ifdim\dimen@>\z@ \dimen@0pt\fi
    \rel@kern{0.6}\kern-\dimen@
    \if#31
      \overline{\rel@kern{-0.6}\kern\dimen@\macc@nucleus\rel@kern{0.4}\kern\dimen@}%
      \advance\dimen@0.4\dimexpr\macc@kerna
      \let\final@kern#2%
      \ifdim\dimen@<\z@ \let\final@kern1\fi
      \if\final@kern1 \kern-\dimen@\fi
    \else
      \overline{\rel@kern{-0.6}\kern\dimen@#1}%
    \fi
  }%
  \macc@depth\@ne
  \let\math@bgroup\@empty \let\math@egroup\macc@set@skewchar
  \mathsurround\z@ \frozen@everymath{\mathgroup\macc@group\relax}%
  \macc@set@skewchar\relax
  \let\mathaccentV\macc@nested@a
  \if#31
    \macc@nested@a\relax111{#1}%
  \else
    \def\gobble@till@marker##1\endmarker{}%
    \futurelet\first@char\gobble@till@marker#1\endmarker
    \ifcat\noexpand\first@char A\else
      \def\first@char{}%
    \fi
    \macc@nested@a\relax111{\first@char}%
  \fi
  \endgroup
}
\makeatother

\newif\ifmlj
\mljfalse

\DeclareMathOperator*{\E}{E}
\DeclareMathOperator*{\Var}{Var}
\DeclareMathOperator*{\Cov}{Cov}

\DeclareMathOperator*{\argmin}{argmin}
\DeclareMathOperator*{\sign}{sign}

\renewcommand{\Re}{\mathbb{R}}

\newcommand{\AUC}{\mathrm{AUC}}

\newcommand{\PN}{\mathrm{PN}}
\newcommand{\PU}{\mathrm{PU}}
\newcommand{\NU}{\mathrm{NU}}
\newcommand{\PNPU}{\mathrm{PNPU}}
\newcommand{\PNNU}{\mathrm{PNNU}}
\newcommand{\PNU}{\mathrm{PNU}}

\newcommand{\PP}{\mathrm{PP}}
\newcommand{\NN}{\mathrm{NN}}

\newcommand{\rP}{{\mathrm{P}}}
\newcommand{\rN}{{\mathrm{N}}}
\newcommand{\rU}{{\mathrm{U}}}
\newcommand{\Ep}{\E\nolimits_{\rP}}
\newcommand{\En}{\E\nolimits_{\rN}}
\newcommand{\Eu}{\E\nolimits_{\rU}}
\newcommand{\pp}{p_{\rP}}
\newcommand{\pn}{p_{\rN}}
\newcommand{\np}{{n_{\rP}}}
\newcommand{\nn}{{n_{\rN}}}
\newcommand{\nun}{{n_{\rU}}}

\newcommand{\Epb}{\E\nolimits_\mathrm{\widebar{P}}}
\newcommand{\Enb}{\E\nolimits_\mathrm{\widebar{N}}}

\newcommand{\nl}{{n_\mathrm{L}}}
\newcommand{\fh}{\widehat{f}}

\newcommand{\thetap}{\theta_{\rP}}
\newcommand{\thetan}{\theta_{\rN}}

\newcommand{\ella}{\ell_{\mathrm{A}}}
\newcommand{\elle}{\ell_{\mathrm{E}}}
\newcommand{\ellh}{\ell_{\mathrm{H}}}
\newcommand{\elll}{\ell_{\mathrm{L}}}

\newcommand{\ells}{\ell_{\mathrm{S}}}

\newcommand{\ellzo}{\ell_{0\textrm{-}1}}

\newcommand{\Rh}{\widehat{R}}

\newcommand{\bw}{\boldsymbol{w}}
\newcommand{\bx}{\boldsymbol{x}}

\newcommand{\bI}{\boldsymbol{I}}

\newcommand{\bmu}{\boldsymbol{\mu}}

\newcommand{\bphi}{\boldsymbol{\phi}}

\newcommand{\bbarphi}{\boldsymbol{\bar{\phi}}}

\newcommand{\bone}{\boldsymbol{1}}

\newcommand{\Eq}{\E\nolimits_{\bx\sim q}}
\newcommand{\Eqp}{\E\nolimits_{\bx'\sim q'}}

\newcommand{\bxp}{\boldsymbol{x}^\rP}
\newcommand{\bxn}{\boldsymbol{x}^\rN}
\newcommand{\bxu}{\boldsymbol{x}^\rU}
\newcommand{\bPhip}{\boldsymbol{\Phi}_{\rP}}
\newcommand{\bPhin}{\boldsymbol{\Phi}_{\rN}}
\newcommand{\bPhiu}{\boldsymbol{\Phi}_{\rU}}

\newcommand{\cF}{\mathcal{F}}

\newcommand{\cO}{\mathcal{O}}
\newcommand{\cXp}{\mathcal{X}_\rP}
\newcommand{\cXn}{\mathcal{X}_\rN}
\newcommand{\cXu}{\mathcal{X}_\rU}

\newcommand{\bhh}{\boldsymbol{\widehat{h}}}
\newcommand{\bHh}{\boldsymbol{\widehat{H}}}

\newcommand{\bmup}{\boldsymbol{\mu}_\mathrm{P}}
\newcommand{\bmun}{\boldsymbol{\mu}_\mathrm{N}}
\newcommand{\rhop}{\rho_\mathrm{P}}
\newcommand{\rhon}{\rho_\mathrm{N}}
\newcommand{\taup}{\tau_\mathrm{P}}
\newcommand{\taun}{\tau_\mathrm{N}}
\newcommand{\ap}{a_\mathrm{P}}
\newcommand{\an}{a_\mathrm{N}}
\newcommand{\bp}{b_\mathrm{P}}
\newcommand{\bn}{b_\mathrm{N}}

\newcommand{\bbarxp}{\boldsymbol{\widebar{x}}^\rP}
\newcommand{\bbarxn}{\boldsymbol{\widebar{x}}^\rN}

\newcommand{\psipn}{\psi_\mathrm{PN}}
\newcommand{\psipp}{\psi_\mathrm{PP}}
\newcommand{\psipu}{\psi_\mathrm{PU}}
\newcommand{\psinn}{\psi_\mathrm{NN}}
\newcommand{\psinu}{\psi_\mathrm{NU}}
\newcommand{\sigmapn}{\sigma_\mathrm{PN}}
\newcommand{\sigmapp}{\sigma_\mathrm{PP}}
\newcommand{\sigmann}{\sigma_\mathrm{NN}}
\newcommand{\taupnpu}{\tau_\mathrm{PN,PU}}
\newcommand{\taupnnu}{\tau_\mathrm{PN,NU}}
\newcommand{\taupnpp}{\tau_\mathrm{PN,PP}}
\newcommand{\taupnnn}{\tau_\mathrm{PN,NN}}
\newcommand{\taupupp}{\tau_\mathrm{PU,PP}}
\newcommand{\taununn}{\tau_\mathrm{NU,NN}}

\usepackage{authblk}
\title{
	Semi-Supervised AUC Optimization 
	based on Positive-Unlabeled Learning
	}
		
\author[1,2]{Tomoya Sakai}
\author[1,2]{Gang Niu}
\author[2,1]{Masashi Sugiyama}

\affil[1]{Graduate School of Frontier Sciences, \protect\\ 
	The University of Tokyo, Japan}
\affil[2]{Center for Advanced Intelligence Project, \protect\\ 
	RIKEN, Japan}

\date{}

\begin{document}
\sloppy

\maketitle

\begin{abstract}
Maximizing the area under the receiver operating characteristic curve (AUC)
is a standard approach to imbalanced classification.
So far, various supervised AUC optimization methods have been developed 
and they are also extended to semi-supervised scenarios
to cope with small sample problems.
However, existing semi-supervised AUC optimization methods rely on
strong distributional assumptions, which are rarely satisfied 
in real-world problems.
In this paper, we propose a novel semi-supervised AUC optimization method
that does not require such restrictive assumptions.
We first develop an AUC optimization method based only on 
positive and unlabeled data (PU-AUC) and then extend it to semi-supervised
learning by combining it with a supervised AUC optimization method.
We theoretically prove that, without the restrictive distributional 
assumptions, unlabeled data contribute to improving the generalization
performance in PU and semi-supervised AUC optimization methods.
Finally, we demonstrate the practical usefulness of the proposed methods
through experiments. 
\end{abstract}

\section{Introduction}
\label{sec:intro}
Maximizing the \emph{area under the receiver operating characteristic curve} 
(AUC) \citep{Radiology:Hanley+McNeil:1982} is a standard approach to imbalanced 
classification \citep{NIPS:Cortes+Mohri:2004}.
While the misclassification rate relies on the sign of the score of a single
sample, AUC is governed by the ranking of the scores of two samples.
Based on this principle, various supervised methods for directly optimizing 
AUC have been developed so far and demonstrated to be useful 
\citep{ICML:Herschtal+Raskutti:2004,ICML:Zhao+etal:2011,ICML:Rakhlin+etal:2012,ICML:Kotlowski+etal:2011,NIPS:Ying+etal:2016}.

However, collecting labeled samples is often expensive and laborious 
in practice.
To mitigate this problem, \emph{semi-supervised} AUC optimization methods 
have been developed that can utilize unlabeled samples
\citep{SIGIR:Amini+etal:2008,ICDM:Fujino+Ueda:2016}.
These semi-supervised methods solely rely on the assumption that an unlabeled
sample that is ``similar'' to a labeled sample shares the same label.
However, such a restrictive distributional assumption (which is often 
referred to as the cluster or the entropy minimization principle)
is rarely satisfied in practice and thus the practical usefulness of
these semi-supervised methods is limited
\citep{ICML:Cozman+etal:2003,ICML:Sokolovska+etal:2008,PAMI:Li+Zhou:2015,PR:Krijthe:2017}.

On the other hand, it has been recently shown that unlabeled data can be 
effectively utilized without such restrictive distributional assumptions
in the context of \emph{classification from positive and unlabeled data}
(PU classification) \citep{NIPS:duPlessis+etal:2014}.
Furthermore, based on recent advances in PU classification 
\citep{NIPS:duPlessis+etal:2014,ICML:duPlessis+etal:2015,NIPS:Niu+etal:2016},
a novel semi-supervised classification approach has been developed that
combines supervised classification with PU classification
\citep{ICML:Sakai+etal:2017}.
This approach inherits the advances of PU classification
that the restrictive distributional assumptions are not necessary
and is demonstrated to perform excellently in experiments.

Following this line of research, we first develop an AUC optimization method
from positive and unlabeled data (PU-AUC) in this paper.
Previously, a \emph{pairwise ranking} method for PU data has been developed
\citep{CIKM:Sundararajan+etal:2011},
which can be regarded as an AUC optimization method for PU data.
However, it merely regards unlabeled data as negative data and 
thus the obtained classifier is biased.
On the other hand, our PU-AUC method is unbiased
and we theoretically prove that unlabeled data contribute to 
reducing an upper bound on the generalization error with 
the optimal parametric convergence rate without 
the restrictive distributional assumptions.

Then we extend our PU-AUC method to the semi-supervised setup by 
combining it with a supervised AUC optimization method.
Theoretically, we again prove that unlabeled data contribute to 
reducing an upper bound on the generalization error 
with the optimal parametric convergence rate 
without the restrictive distributional assumptions,
and further we prove that the variance of the empirical risk 
of our semi-supervised AUC optimization method
can be smaller than that of the plain supervised counterpart.
The latter claim suggests that the proposed semi-supervised
empirical risk is also useful in the cross-validation phase.
Finally, we experimentally demonstrate the usefulness of 
the proposed PU and semi-supervised AUC optimization methods.

\section{Preliminary}
We first describe our problem setting and 
review an existing supervised AUC optimization method.

Let covariate $\bx\in\Re^d$ and its corresponding label $y\in\{\pm1\}$ 
be equipped with 
probability density $p(\bx,y)$, where $d$ is a positive integer.
Suppose we have sets of positive and negative samples:
\begin{align*}
\cXp&:=\{\bxp_i\}^\np_{i=1} 
	\stackrel{\mathrm{i.i.d.}}{\sim} 
	\pp(\bx):=p(\bx\mid y=+1) , \mathrm{\;and} \notag \\ 
\cXn&:=\{\bxn_j\}^\nn_{j=1} 
	\stackrel{\mathrm{i.i.d.}}{\sim} 
	\pn(\bx):=p(\bx\mid y=-1) .
\end{align*}
Furthermore, let $g\colon\Re^d\to\Re$ be a decision function 
and classification is carried out based on its sign:
$\widehat{y}=\sign(g(\bx))$.

The goal is to train a classifier $g$ by maximizing the AUC
\citep{Radiology:Hanley+McNeil:1982,NIPS:Cortes+Mohri:2004} 
defined and expressed as
\begin{align}
\AUC(g)&:=\Ep[\En[I(g(\bxp)\geq g(\bxn))]] \notag \\
&\phantom{:}=1-\Ep[\En[I(g(\bxp)<g(\bxn))]] \notag \\
&\phantom{:}=1-\Ep[\En[\ellzo(g(\bxp)-g(\bxn))]] , 
	\label{eq:auc-def}
\end{align}
where 
$\Ep$ and $\En$ be the expectations over 
$\pp(\bx)$ and $\pn(\bx)$, respectively.
$I(\cdot)$ is the indicator function, 
which is replaced with the \emph{zero-one} loss,
$\ellzo(m)=(1-\sign(m))/2$, to obtain the last equation.
Let
\begin{align*}
f(\bx,\bx'):=g(\bx)-g(\bx')
\end{align*} 
be a composite classifier.
Maximizing the AUC corresponds to minimizing the second term 
in Eq.\eqref{eq:auc-def}.
Practically, to avoid the discrete nature of the zero-one loss,
we replace the zero-one loss with a surrogate loss $\ell(m)$
and consider 
the following PN-AUC risk
\citep{ICML:Herschtal+Raskutti:2004,ICML:Kotlowski+etal:2011,ICML:Rakhlin+etal:2012}:
\begin{align}
R_\PN(f):=\Ep[\En[\ell(f(\bxp,\bxn))]] .
\label{eq:pn-risk}
\end{align}
In practice, we train a classifier by minimizing the empirical PN-AUC risk 
defined as 
\begin{align*}
\Rh_\PN(f):=\frac{1}{\np\nn}\sum^\np_{i=1}\sum^\nn_{j=1}
	\ell(f(\bxp_i,\bxn_j)) .
\end{align*}

Similarly to the \emph{classification-calibrated} loss 
\citep{JASA:Bartlett+etal:2006} in misclassification rate minimization,
the consistency of AUC optimization in terms of loss functions
has been studied recently \citep{IJCAI:Gao+Zhou:2015,AIJ:Gao+etal:2016}.
They showed that minimization of the AUC risk with a consistent loss function
is asymptotically equivalent to that with the zero-one loss function. 
The squared loss $\ells(m):=(1-m)^2$, 
the exponential loss $\elle(m):=\exp(-m)$,
and 
the logistic loss $\elll(m):=\log(1+\exp(-m))$
are shown to be consistent, while the hinge loss $\ellh(m):=\max(0,1-m)$
and the absolute loss $\ella(m):=|1-m|$ are \emph{not} 
consistent.

\section{Proposed Method}
\label{sec:pu-auc}
In this section, we first propose an AUC optimization method from
positive and unlabeled data 
and then extend it to a semi-supervised AUC optimization method.

\subsection{PU-AUC Optimization}
In PU learning, we do not have negative data while 
we can use unlabeled data drawn from marginal density $p(\bx)$
in addition to positive data:
\begin{align}
\cXu&:=\{\bxu_k\}^\nun_{k=1} 
	\stackrel{\mathrm{i.i.d.}}{\sim} 
 	p(\bx)=\thetap\pp(\bx)+\thetan\pn(\bx) ,
 	\label{eq:marginal}
\end{align}
where
\begin{align*}
\thetap:=p(y=+1) \mathrm{\;\;and\;\;} \thetan:=p(y=-1).
\end{align*}
We derive an equivalent expression to the PN-AUC risk 
that depends only on positive and unlabeled data distributions 
without the negative data distribution. 
In our derivation and theoretical analysis, 
we assume that $\thetap$ and $\thetan$ are known.
In practice, they are replaced by their estimate obtained, e.g., 
by \citet{MLJ:duPlessis+etal:2017}, \citet{IEICE:Kawakubo+etal:2016},
and references therein.

From the definition of the marginal density in Eq.~\eqref{eq:marginal}, 
we have
\begin{align*}
\Ep[\Eu[\ell(f(\bxp, \bxu))]]  
&=\thetap\Ep[\Epb[\ell(f(\bxp, \bbarxp))]] 	
	+\thetan\Ep[\En[\ell(f(\bxp, \bxn))]]  \\
&=\thetap\Ep[\Epb[\ell(f(\bxp, \bbarxp))]] 	
	+\thetan R_\PN(f),
\end{align*}
where $\Epb$ denotes the expectation over $\pp(\bbarxp)$.
Dividing the above equation by $\thetan$ and rearranging it,
we can express the PN-AUC risk in Eq.~\eqref{eq:pn-risk} 
based on PU data (the PU-AUC risk) as
\begin{align}
R_\PN(f)=\frac{1}{\thetan}
	\Ep[\Eu[\ell(f(\bxp, \bxu))]] 	
	-\frac{\thetap}{\thetan}
	\Ep[\Epb[\ell(f(\bxp, \bbarxp))]] := R_\PU(f) .
\label{eq:pu-risk}	
\end{align}
We refer to the method minimizing the PU-AUC risk as 
\emph{PU-AUC optimization}.
We will theoretically investigate the superiority of $R_\PU$
in Section~\ref{sec:theory-gen-err}.

To develop a semi-supervised AUC optimization method later,
we also consider AUC optimization from \emph{negative} and unlabeled data,
which can be regarded as a mirror of PU-AUC optimization.
From the definition of the marginal density in Eq.~\eqref{eq:marginal},
we have
\begin{align*}
\Eu[\En[\ell(f(\bxu, \bxn))]]  
&=\thetap\Ep[\En[\ell(f(\bxp, \bxn))]] 
	+\thetan\En[\Enb[\ell(f(\bxn, \bbarxn))]] \\
&=\thetap R_\PN(f)
	+\thetan\En[\Enb[\ell(f(\bxn, \bbarxn))]] , 
\end{align*}
where $\Enb$ denotes the expectation over $\pn(\bbarxn)$.
Rearranging the above equation, we can obtain 
the PN-AUC risk in Eq.~\eqref{eq:pn-risk} based on 
negative and unlabeled data (the NU-AUC risk):
\begin{align}
R_\PN(f)=\frac{1}{\thetap}\!
	\Eu[\En[\ell(f(\bxu, \bxn))]] 
	-\frac{\thetan}{\thetap}\! 
	\En[\Enb[\ell(f(\bxn, \bbarxn))]] := R_\NU(f) .
\label{eq:nu-risk}		
\end{align}
We refer to the method minimizing the NU-AUC risk as
\emph{NU-AUC optimization}.

\subsection{Semi-Supervised AUC Optimization}
\label{sec:ss-auc}
Next, we propose a novel semi-supervised AUC optimization method
based  on positive-unlabeled learning.
The idea is to combine the PN-AUC risk with the PU-AUC/NU-AUC risks,
similarly to \citet{ICML:Sakai+etal:2017}.\footnote{
In \citet{ICML:Sakai+etal:2017}, the combination of the PU and NU risks
has also considered and found to be less favorable than 
the combination of the PN and PU/NU risks. 
For this reason, we focus on the latter in this paper. 
}

First of all, let us define the PNPU-AUC and PNNU-AUC risks as
\begin{align*}
R_\PNPU^\gamma(f)
	&:=(1-\gamma)R_\PN(f)+\gamma R_\PU(f), \\
R_\PNNU^\gamma(f)
	&:=(1-\gamma)R_\PN(f)+\gamma R_\NU(f) ,
\end{align*}
where $\gamma\in[0, 1]$ is the combination parameter.
We then define the PNU-AUC risk as
\begin{align}
R_\PNU^\eta(f):=
\begin{cases}
R_\PNPU^\eta(f) & (\eta\geq0), \\
R_\PNNU^{-\eta}(f) & (\eta<0),
\end{cases}
\label{eq:pnu-risk}
\end{align}
where $\eta\in[-1, 1]$ is the combination parameter.
We refer to the method minimizing the PNU-AUC risk
as \emph{PNU-AUC optimization}.
We will theoretically discuss the superiority of $R_\PNPU^\gamma$
and $R_\PNNU^\gamma$ in Section~\ref{sec:theory-gen-err}.

\subsection{Discussion about Related Work}
\citet{CIKM:Sundararajan+etal:2011} proposed 
a pairwise ranking method for PU data, 
which can be regarded as an AUC optimization method for PU data.
Their approach simply regards unlabeled data as negative data and the ranking
SVM \citep{KDD:Joachims:2002} 
is applied to PU data so that the score of positive data 
tends to be higher than that of unlabeled data.  
Although this approach is simple and shown computationally
efficient in experiments, the obtained classifier is biased.
From the mathematical viewpoint, the existing method 
ignores the second term in Eq.~\eqref{eq:pu-risk}   
and maximizes only the first term with the hinge loss function.
However, the effect of ignoring the second term 
is not negligible when the class prior, $\thetap$, is
not sufficiently small.
In contrast, our proposed PU-AUC risk includes 
the second term so that the PU-AUC risk is equivalent to the PN-AUC risk.
 
Our semi-supervised AUC optimization method 
can be regarded as an extension of the work by
\citet{ICML:Sakai+etal:2017}.
They considered the misclassification rate
as a measure to train a classifier
and proposed a semi-supervised classification method 
based on the recently proposed PU classification method
\citep{NIPS:duPlessis+etal:2014,ICML:duPlessis+etal:2015}.
On the other hand, we train a classifier by maximizing the AUC, 
which is a standard approach for imbalanced classification.
To this end, we first developed an AUC optimization method for PU data,
and then extended it to a semi-supervised AUC optimization method.
Thanks to the AUC maximization formulation, our proposed method is expected to
perform better than the method proposed by \citet{ICML:Sakai+etal:2017} 
for imbalanced data sets.

\section{Theoretical Analyses}
\label{sec:theory}
In this section, we theoretically analyze the proposed risk functions.
We first derive generalization error bounds of our methods and 
then discuss variance reduction.

\subsection{Generalization Error Bounds}
\label{sec:theory-gen-err}
Recall the composite classifier $f(\bx,\bx')=g(\bx)-g(\bx')$.
As the classifier $g$,  
we assume the linear-in-parameter model given by
\begin{align*}
g(\bx)=\sum^b_{\ell=1}w_\ell \phi(\bx)=\bw^\top\bphi(\bx) ,
\end{align*}
where ${}^\top$ denotes the transpose of vectors and matrices,
$b$ is the number of basis functions, 
$\bw=(w_1,\ldots,w_b)^\top$ is a parameter vector, 
and $\bphi(\bx)=(\phi_1(\bx),\ldots,\phi_b(\bx))^\top$ 
is a basis function vector.
Let $\cF$ be a function class of bounded hyperplanes:
\begin{align*}
\cF:=\{ f(\bx,\bx')=
	\bw^\top(\bphi(\bx)-\bphi(\bx'))
	\mid \|\bw\|\leq C_{\bw}; \; 
	\forall\bx\colon
	\|\bphi(\bx)\|\leq C_{\bphi} \},	
\end{align*}
where 
$C_{\bw}>0$ and $C_{\bphi}>0$ are certain positive constants.
This assumption is reasonable because
the $\ell_2$-regularizer included in training
and the use of bounded basis functions,
e.g., the Gaussian kernel basis,
ensure that the minimizer of the empirical AUC risk belongs to such 
the function class $\cF$.  
We assume that a surrogate loss is bounded from above by $C_\ell$
and denote the Lipschitz constant by $L$.
For simplicity,\footnote{
Our theoretical analysis can be easily extended to 
the loss satisfying $\ellzo(m)\leq M\ell(m)$ with a certain $M>0$.
}
we focus on a surrogate loss satisfying $\ellzo(m)\leq\ell(m)$.
For example, the squared loss
and the exponential loss satisfy the condition.\footnote{
These losses are bounded in our setting, 
since the input to $\ell(m)$, i.e., $f$ is bounded. 
}

Let
\begin{align*}
I(f)=\Ep[\En[\ellzo(f(\bxp,\bxn))]] 
\end{align*} 
be the generalization error of $f$ in AUC optimization.
For convenience, we define
\begin{align*}
h(\delta)&:=2\sqrt{2}LC_\ell C_{\bw}C_{\bphi}
	+\frac{3}{2}\sqrt{2\log(2/\delta)} .
\end{align*} 
In the following, we prove the generalization error bounds
of both PU and semi-supervised AUC optimization methods.

For the PU-AUC/NU-AUC risks,
we prove the following generalization error bounds
(its proof is available in Appendix~\ref{proof:gen-err}):
\begin{theorem}
\label{thm:pu-nu-gen-err}
For any $\delta>0$, the following inequalities hold separately
with probability at least $1-\delta$ for all $f\in\cF$:
\begin{align*}
I(f)\leq \Rh_\PU(f) + h(\delta/2)
	\Big(\frac{1}{\thetan\sqrt{\min(\np,\nun)}}
	+\frac{\thetap}{\thetan\sqrt{\np}}
	\Big) , \\
I(f)\leq \Rh_\NU(f) + h(\delta/2)
	\Big(\frac{1}{\thetap\sqrt{\min(\nn,\nun)}}
	+\frac{\thetan}{\thetap\sqrt{\nn}} 
	\Big) , \\
\end{align*}
where $\Rh_\PU$ and $\Rh_\NU$ are 
unbiased empirical risk estimators corresponding to 
$R_\PU$ and $R_\NU$, respectively.
\end{theorem}
Theorem \ref{thm:pu-nu-gen-err} guarantees that 
$I(f)$ can be bounded from above by the empirical risk, 
$\Rh(f)$, plus the confidence terms of order 
\begin{align*}
\cO_p\Big(\frac{1}{\sqrt{\np}}+\frac{1}{\sqrt{\nun}}\Big)  
~~\mathrm{and}~~	
\cO_p\Big(\frac{1}{\sqrt{\nn}}+\frac{1}{\sqrt{\nun}}\Big) .
\end{align*}
Since $\np$ ($\nn$) and $\nun$ can increase independently 
in our setting, 
this is the optimal convergence rate without 
any additional assumptions
\citep{book:Vapnik:1998,TIT:Mendelson:2008}.

For the PNPU-AUC and PNNU-AUC risks,
we prove the following generalization error bounds
(its proof is also available in Appendix~\ref{proof:gen-err}):
\begin{theorem}
\label{thm:punu-pnu-gen-err}
For any $\delta>0$, the following inequalities hold separately
with probability at least $1-\delta$ for all $f\in\cF$: 
\begin{align*}
I(f)&\leq\Rh_\PNPU^\gamma(f)+h(\delta/3)
	\Big( 
	\frac{1-\gamma}{\sqrt{\min(\np,\nn)}}
	+\frac{\gamma}{\thetan\sqrt{\min(\np,\nun)}}
	+\frac{\gamma\thetap}{\thetan\sqrt{\np}} 
	\Big) , \\	
I(f)&\leq\Rh_\PNNU^\gamma(f)+h(\delta/3)
	\Big(
	\frac{1-\gamma}{\sqrt{\min(\np,\nn)}}
	+\frac{\gamma}{\thetap\sqrt{\min(\nn,\nun)}}
	+\frac{\gamma\thetan}{\thetap\sqrt{\nn}} 
	\Big) .
\end{align*}
where $\Rh_\PNPU^\gamma$ and $\Rh_\PNNU^\gamma$ 
are unbiased empirical risk estimators corresponding to
$R_\PNPU^\gamma$ and $R_\PNNU^\gamma$, respectively.
\end{theorem}
Theorem \ref{thm:punu-pnu-gen-err} guarantees that 
$I(f)$ can be bounded from above by the empirical risk, 
$\Rh(f)$, plus the confidence terms of order
\begin{align*}
\cO_p\Big(\frac{1}{\sqrt{\np}}+\frac{1}{\sqrt{\nn}}
	+\frac{1}{\sqrt{\nun}}\Big) .
\end{align*}
Again, since $\np$, $\nn$, and $\nun$ can increase independently
in our setting, this is the optimal convergence rate
without any additional assumptions.

\subsection{Variance Reduction}
\label{sec:theory-var}
In the existing semi-supervised classification method based on PU learning, 
the variance of the empirical risk was proved to be smaller 
than the supervised counterpart 
under certain conditions \citep{ICML:Sakai+etal:2017}.
Similarly, we here investigate
if the proposed semi-supervised risk estimators 
have smaller variance than its supervised counterpart.

Let us introduce the following variances and 
covariances:\footnote{
$\Var\nolimits_\PN$, $\Var\nolimits_\mathrm{P\widebar{P}}$,
and $\Var\nolimits_\mathrm{N\widebar{N}}$ 
are the variances over $\pp(\bxp)\pn(\bxn)$,
$\pp(\bxp)\pp(\bbarxp)$, 
and $\pn(\bxn)\pn(\bbarxn)$, respectively.
$\Cov\nolimits_\mathrm{PN,P\widebar{P}}$,
$\Cov\nolimits_\mathrm{PN,N\widebar{N}}$, 
$\Cov\nolimits_\mathrm{PU,P\widebar{P}}$,
and $\Cov\nolimits_\mathrm{NU,N\widebar{N}}$ 
are the covariances over $\pp(\bxp)\pn(\bxn)\pp(\bbarxp)$,
$\pp(\bxp)\pn(\bxn)\pn(\bbarxn)$,
$\pp(\bxp)p(\bxu)\pp(\bbarxp)$,
and $\pn(\bxn)p(\bxu)\pn(\bbarxn)$, respectively.
}
\begin{align*}
\sigmapn^2(f)&=\Var\nolimits_\mathrm{PN}
	[\ell(f(\bxp,\bxn))], \\
\sigmapp^2(f)&=\Var\nolimits_\mathrm{P\widebar{P}}
	[\ell(f(\bxp,\bbarxp))], \\
\sigmann^2(f)&=\Var\nolimits_\mathrm{N\widebar{N}} 
	[\ell(f(\bxn,\bbarxn))], \\
\taupnpp(f)&=\Cov\nolimits_\mathrm{PN,P\widebar{P}}
 	[\ell(f(\bxp,\bxn)), \ell(f(\bxp,\bbarxp))] , \\
\taupnnn(f)&=\Cov\nolimits_\mathrm{PN,N\widebar{N}}
 	[\ell(f(\bxp,\bxn)), \ell(f(\bxn,\bbarxn))] , \\ 
\taupupp(f)&=\Cov\nolimits_\mathrm{PU,P\widebar{P}}
 	[\ell(f(\bxp,\bxu)), \ell(f(\bxp,\bbarxp))] , \\
\taununn(f)&=\Cov\nolimits_\mathrm{NU,N\widebar{N}}
 	[\ell(f(\bxp,\bxu)), \ell(f(\bxn,\bbarxn))] .
\end{align*}

Then, we have the following theorem
(its proof is available in Appendix~\ref{proof:var}):
\begin{theorem}
\label{thm:var-red-pnpu-pnnu}
Assume $\nun\to\infty$.
For any fixed $f$,
the minimizers of the variance of the empirical PNPU-AUC and PNNU-AUC risks
are respectively obtained by
\begin{align}
\gamma_\PNPU&=\argmin_{\gamma} \Var[\Rh_\PNPU^\gamma(f)] 
=\frac{\psipn-\psipp/2}{\psipn+\psipu-\psipp} , 
\label{eq:gam-pnpu} \\
\gamma_\PNNU&=\argmin_{\gamma} \Var[\Rh_\PNNU^\gamma(f)]
=\frac{\psipn-\psinn/2}{\psipn+\psinu-\psinn} ,
\label{eq:gam-pnnu}
\end{align}
where 
\begin{align*}
\psipn&=\frac{1}{\np\nn}\sigmapn^2(f),   \\
\psipu&=\frac{\thetap^2}{\thetan^2\np^2}\sigmapp^2(f)
	-\frac{\thetap}{\thetan^2\np}\taupupp(f) , \\
\psipp&=\frac{1}{\thetan\np}\taupnpu(f)
	-\frac{\thetap}{\thetan\np}\taupnpp(f), \\
\psinu&=\frac{\thetan^2}{\thetap^2\nn^2}\sigmann^2(f)
	-\frac{\thetan}{\thetap^2\nn}\taununn(f) , \\
\psinn&=\frac{1}{\thetap\nn}\taupnnu(f) 
	-\frac{\thetan}{\thetap\nn}\taupnnn(f) . 
\end{align*}
Additionally, we have
$\Var[\Rh_\PNPU^\gamma(f)]<\Var[\Rh_\PN(f)]$
for any $\gamma\in(0,2\gamma_\PNPU)$ 
if $\psipn+\psipu>\psipp$ and $2\psipn>\psipp$.
Similarly, we have
$\Var[\Rh_\PNNU^\gamma(f)]<\Var[\Rh_\PN(f)]$
for any $\gamma\in(0,2\gamma_\PNNU)$
if $\psipn+\psinu>\psinn$ and $2\psipn>\psinn$.
\end{theorem}
This theorem means that, if $\gamma$ is chosen appropriately, 
our proposed risk estimators, $\Rh_\PNPU^\gamma$ and $\Rh_\PNNU^\gamma$,
have smaller variance than the standard supervised risk estimator
$\Rh_\PN$.
A practical consequence of Theorem~\ref{thm:var-red-pnpu-pnnu} 
is that when we conduct cross-validation for hyperparameter selection,  
we may use our proposed risk estimators 
$\Rh_\PNPU^\gamma$ and $\Rh_\PNNU^\gamma$ 
instead of the standard supervised risk
estimator $\Rh_\PN$ since they are more stable 
(see Section~\ref{sec:impl-cv} for details).

\section{Practical Implementation}
\label{sec:impl}
In this section, we explain the implementation details of 
our proposed methods.

\subsection{General Case}
In practice, 
the AUC risks $R$ introduced above are replaced with
their empirical version $\Rh$, where the expectations in $R$ are 
replaced with the corresponding sample averages. 

Here, we focus on the linear-in-parameter model given by
\begin{align*}
g(\bx)=\sum^b_{\ell=1}w_\ell \phi(\bx)=\bw^\top\bphi(\bx) ,
\end{align*}
where ${}^\top$ denotes the transpose of vectors and matrices,
$b$ is the number of basis functions, 
$\bw=(w_1,\ldots,w_b)^\top$ is a parameter vector, 
and $\bphi(\bx)=(\phi_1(\bx),\ldots,\phi_b(\bx))^\top$ 
is a basis function vector.
The linear-in-parameter model allows us to 
express the composite classifier as
\begin{align*}
f(\bx,\bx')=\bw^\top\bbarphi(\bx,\bx') ,
\end{align*} 
where 
\begin{align*}
\bbarphi(\bx,\bx'):=\bphi(\bx)-\bphi(\bx')
\end{align*}
is a composite basis function vector.
We train the classifier by minimizing the $\ell_2$-regularized 
empirical AUC risk:
\begin{align*}
\min_{\bw} \Rh(f) + \lambda \|\bw\|^2 , 
\end{align*} 
where $\lambda\geq 0$ is the regularization parameter.

\subsection{Analytical Solution for Squared Loss}
For the squared loss $\ell_S(m):=(1-m)^2$, 
the empirical PU-AUC risk\footnote{We discuss the way of estimating 
the PU-AUC risk in Appendix~\ref{app:emp_pu_risk}. 
} can be expressed as
\begin{align*}
\Rh_\PU(f)&=\frac{1}{\thetan\np\nun}\sum^\np_{i=1}\sum^\nun_{k=1}
	\ells(f(\bxp_i,\bxu_k)) \\ 
	&\phantom{=}-\frac{\thetap}{\thetan\np(\np-1)}
	\sum^\np_{i=1}\sum^\np_{i'=1}\ells(f(\bxp_i,\bxp_{i'}))
	+ \frac{\thetap}{\thetan(\np-1)}\\
	&=1-2\bw^\top\bhh_\PU + \bw^\top\bHh_\PU\bw 
		- \bw^\top\bHh_\PP\bw ,	
\end{align*} 
where 
\begin{align*}
\bhh_\PU&:=\frac{1}{\thetan\np}\bPhip^\top\bone_\np 
	- \frac{1}{\thetan\nun}\bPhiu^\top\bone_\nun , \\
\bHh_\PU&:=\frac{1}{\thetan\np}\bPhip^\top\bPhip
	-\frac{1}{\thetan\np\nun}\bPhiu^\top\bone_\nun\bone_\np^\top\bPhip \\
	&\phantom{:=}
	-\frac{1}{\thetan\np\nun}\bPhip^\top\bone_\np\bone_\nun^\top\bPhiu
	+\frac{1}{\thetan\nun}\bPhiu^\top\bPhiu , \\
\bHh_\PP&:=\frac{2\thetap}{\thetan(\np-1)}\bPhip^\top\bPhip 
	- \frac{2\thetap}{\thetan\np(\np-1)}
	\bPhip^\top\bone_\np\bone_\np^\top\bPhip , \\
\bPhip&:=(\bphi(\bxp_1), \ldots, \bphi(\bxp_\np))^\top , \\	
\bPhiu&:=(\bphi(\bxu_1), \ldots, \bphi(\bxu_\nun))^\top ,
\end{align*}
and $\bone_b$ is the $b$-dimensional vector whose elements are all one.
With the $\ell_2$-regularizer,
we can analytically obtain the solution by
\begin{align*}
\widehat{\bw}_\PU
	:=(\bHh_\PU-\bHh_\PP+\lambda\bI_b)^{-1}\bhh_\PU ,
\end{align*}
where $\bI_b$ is the $b$-dimensional identity matrix.

The computational complexity of 
computing $\bhh_\PU$, $\bHh_\PU$, and $\bHh_\PP$ are
$\cO((\np+\nun)b)$,  $\cO((\np+\nun)b^2)$, and $\cO(\np b^2)$, respectively.
Then, solving a system of linear equations 
to obtain the solution $\widehat{\bw}_\PU$ 
requires the computational complexity of $\cO(b^3)$. 
In total,
the computational complexity of this PU-AUC optimization method is 
$\cO((\np+\nun)b^2+b^3)$.

As given by Eq.~\eqref{eq:pnu-risk}, 
our PNU-AUC optimization method consists of the
PNPU-AUC risk and the PNNU-AUC risk. 
For the squared loss $\ell_S(m):=(1-m)^2$, 
the empirical PNPU-AUC risk can be expressed as 
\begin{align*}
\Rh_\PNPU^\gamma(f)&=\frac{1-\gamma}{\np\nn}\sum^\np_{i=1}\sum^\nn_{j=1}
	\ells(f(\bxp_i,\bxn_j)) 
	+\frac{\gamma}{\thetan\np\nun}\sum^\np_{i=1}\sum^\nun_{k=1}
	\ells(f(\bxp_i,\bxu_k)) \\	
	&\phantom{=} 
	-\frac{\gamma\thetap}{\thetan\np(\np-1)}\sum^\np_{i=1}\sum^\np_{i'=1}
	\ells(f(\bxp_i,\bxp_{i'})) + \frac{\gamma\thetap}{\thetan(\np-1)} \\
	&=(1-\gamma)-2(1-\gamma)\bw^\top\bhh_\PN 
	+ (1-\gamma)\bw^\top\bHh_\PN\bw \\
	&\phantom{=} 
	+\gamma - 2\gamma\bw^\top\bhh_\PU 
		+ \gamma\bw^\top\bHh_\PU\bw - \gamma\bw^\top\bHh_\PP\bw ,	
\end{align*} 
where 
\begin{align*}
\bhh_\PN&:=\frac{1}{\np}\bPhip^\top\bone_\np 
	- \frac{1}{\nn}\bPhin^\top\bone_\nn , \\
\bHh_\PN&:=\frac{1}{\np}\bPhip^\top\bPhip
	-\frac{1}{\np\nn}\bPhip^\top\bone_\np\bone_\nn^\top\bPhin \\
	&\phantom{:=}	
	-\frac{1}{\np\nn}\bPhin^\top\bone_\nn\bone_\np^\top\bPhip
	+\frac{1}{\nn}\bPhin^\top\bPhin , \\
\bPhin&:=(\bphi(\bxn_1), \ldots, \bphi(\bxn_\nn))^\top . \\	
\end{align*}
The solution for the $\ell_2$-regularized PNPU-AUC optimization
can be analytically obtained by
\begin{align*}
\widehat{\bw}_\PNPU^\gamma
	:=\Big((1-\gamma)\bHh_\PN + \gamma\bHh_\PU - \gamma\bHh_\PP
	+\lambda\bI_b \Big)^{-1}
	\Big((1-\gamma)\bhh_\PN + \gamma\bhh_\PU
	\Big) .
\end{align*}
Similarly, the solution for 
the $\ell_2$-regularized PNNU-AUC optimization can be obtained by
\begin{align*}
\widehat{\bw}_\PNNU^\gamma
	:=\Big((1-\gamma)\bHh_\PN + \gamma\bHh_\NU - \gamma\bHh_\NN
	+\lambda\bI_b \Big)^{-1}
	\Big((1-\gamma)\bhh_\PN + \gamma\bhh_\NU
	\Big) .
\end{align*} 
where 
\begin{align*}
\bhh_\NU&:=\frac{1}{\thetap\nun}\bPhiu^\top\bone_\nun 
	-\frac{1}{\thetap\nn}\bPhin^\top\bone_\nn , \\
\bHh_\NU&:=\frac{\thetan}{\thetap\nn}\bPhin^\top\bPhin
	-\frac{\thetan}{\thetap\nn\nun}\bPhiu^\top\bone_\nun\bone_\nn^\top\bPhin \\
	&\phantom{:=}
	-\frac{\thetan}{\thetap\nn\nun}\bPhin^\top\bone_\nn\bone_\nun^\top\bPhiu
	+\frac{\thetan}{\thetap\nun}\bPhiu^\top\bPhiu , \\
\bHh_\NN&:=\frac{2\thetan}{\thetap(\nn-1)}\bPhin^\top\bPhin 
	- \frac{2\thetan}{\thetap\nn(\nn-1)}
	\bPhin^\top\bone_\nn\bone_\nn^\top\bPhin . \\
\end{align*}

The computational complexity of 
computing $\bhh_\PN$ and $\bHh_\PN$ are
$\cO((\np+\nn)b)$ and  $\cO((\np+\nn)b^2)$, respectively. 
Then, obtaining the solution
$\widehat{\bw}_\PNPU$ ($\widehat{\bw}_\PNNU$)
requires the computational complexity of $\cO(b^3)$.
Including the computational complexity of computing
$\bhh_\PU$, $\bHh_\PU$, and $\bHh_\PP$,
the total computational complexity of the PNPU-AUC optimization method is 
$\cO((\np+\nn+\nun)b^2+b^3)$. 
Similarly, the total computational complexity of the PNNU-AUC optimization
method is $\cO((\np+\nn+\nun)b^2+b^3)$.
Thus, the computational complexity of the PNU-AUC optimization method 
is $\cO((\np+\nn+\nun)b^2+b^3)$.  

From the viewpoint of computational complexity, the squared loss and 
the exponential loss are more efficient than the logistic loss
because these loss functions reduce the nested summations
to individual ones.
More specifically, for example, in the PU-AUC optimization method, 
the logistic loss requires $\cO(\np\nun)$ operations 
for evaluating the first term in the PU-AUC risk, i.e.,
the loss over positive and unlabeled samples.  
In contrast, the squared loss and exponential loss 
reduce the number of operations for 
loss evaluation to $\cO(\np+\nun)$.\footnote{
For example, 
the exponential loss over positive and unlabeled data  
can be computed as follows:
\begin{align*}
\sum^\np_{i=1}\sum^{\nun}_{k=1}\elle(f(\bxp_i,\bxu_k))
&=\sum^\np_{i=1}\sum^{\nun}_{k=1}\exp(-g(\bxp_i)
+g(\bxu_k)) \\
&=\sum^\np_{i=1}\exp(-g(\bxp_i))\sum^{\nun}_{k=1}\exp(g(\bxu_k)).
\end{align*} 
Thus, the number of operations for loss evaluation 
is reduced to $\np+\nun+1$ rather than $\np\nun$.
} This property is beneficial especially when we handle 
large scale data sets.

\subsection{Cross-Validation}
\label{sec:impl-cv}
To tune the hyperparameters such as the regularization parameter $\lambda$,
we use the cross-validation.

For the PU-AUC optimization method, 
we use the PU-AUC risk in Eq.~\eqref{eq:pu-risk} 
with the zero-one loss as the cross-validation score.

For the PNU-AUC optimization method,
we use the PNU-AUC risk in Eq.~\eqref{eq:pnu-risk}
with the zero-one loss as the score.
To this end, however, we need to fix the combination parameter 
$\eta$ in the cross-validation score
in advance and then, we tune the hyperparameters 
including the combination parameter.
More specifically, let $\widebar{\eta}\in\!\![-1,1]$ 
be the predefined combination parameter.
We conduct cross-validation with respect to $R_\PNU^{\widebar{\eta}}(f)$
for tuning the hyperparameters.
Since the PNU-AUC risk is equivalent to the PN-AUC risk 
for any $\widebar{\eta}$, we can choose any $\widebar{\eta}$ in principle.
However, when the empirical PNU-AUC risk is used in practice, 
choice of $\widebar{\eta}$ may affect the performance of cross-validation.

Here, based on the theoretical result of variance reduction given 
in Section~\ref{sec:theory-var},
we give a practical method to determine $\widebar{\eta}$. 
Assuming the covariances, e.g., $\taupnpp(f)$, are small enough 
to be neglected
and $\sigmapn(f)=\sigmapp(f)=\sigmann(f)$,	 
we can obtain a simpler form of Eqs.~\eqref{eq:gam-pnpu} and \eqref{eq:gam-pnnu}
as 
\begin{align*}
\widebar{\gamma}_\PNPU&=
	\frac{1}	{1+\thetap^2\nn/(\thetan^2\np)} ,  \\
\widebar{\gamma}_\PNNU&=
	\frac{1}{1+\thetan^2\np/(\thetap^2\nn)} .	
\end{align*}
They can be computed simply from
the number of samples and the (estimated) class-prior.
Finally, to select the combination parameter $\eta$,
we use $\Rh_\PNPU^{\widebar{\gamma}_\PNPU}$ for $\eta\geq0$,
and $\Rh_\PNNU^{\widebar{\gamma}_\PNNU}$ for $\eta<0$.

\section{Experiments}
\label{sec:exp}
In this section, we numerically investigate the behavior 
of the proposed methods and evaluate their performance on 
various data sets. 
All experiments were carried out using a PC equipped with
two $2.60$GHz Intel\textsuperscript{\textregistered} 
Xeon\textsuperscript{\textregistered} E$5$-$2640$ v$3$ CPUs.

As the classifier,
we used the linear-in-parameter model.
In all experiments except text classification tasks, 
we used the Gaussian kernel basis function expressed as
\begin{align*}
\phi_\ell(\bx)=\exp\Big(-\frac{\|\bx-\bx_\ell\|^2}{2\sigma^2}\Big),
\end{align*}  
where $\sigma>0$ is the Gaussian bandwidth,
$\{\bx_\ell\}^b_{\ell=1}$ are the samples 
randomly selected from training samples $\{\bx_i\}^n_{i=1}$
and $n$ is the number of training samples.
In text classification tasks,  
we used the linear kernel basis function:
\begin{align*}
\phi_\ell(\bx)=\bx^\top\bx_\ell.
\end{align*} 
The number of basis functions was set at $b=\min(n, 200)$.
The candidates of the Gaussian bandwidth were 
$\mathrm{median}(\{\|\bx_i-\bx_j\|\}^n_{i,j=1})
\times\{1/8,1/4,1/2,1,2\}$ and that of the regularization parameter 
were $\{10^{-3},10^{-2},10^{-1},10^0,10^1\}$.
All hyper-parameters were determined by five-fold cross-validation.
As the loss function, we used the squared loss function 
$\ells(m)=(1-m)^2$.

\subsection{Effect of Variance Reduction}
\label{sec:eff-var-red}
First, we numerically confirm the effect of variance reduction.
We compare the variance of the empirical PNU-AUC risk against 
the variance of the empirical PN-AUC risk, 
$\Var[\Rh_\PNU^\eta(f)]$ vs.~$\Var[\Rh_\PN(f)]$, 
under a fixed classifier $f$.

As the fixed classifier, we used the minimizer of 
the empirical PN-AUC risk, denoted by $\widehat{f}_\PN$. 
The number of positive and negative samples for training
varied as $(\np,\nn)=(2,8)$, $(10,10)$, and $(18, 2)$. 
We then computed the variance of the empirical PN-AUC and PNU-AUC risks 
with additional $10$ positive, $10$ negative, and $300$ unlabeled samples.
As the data set,
we used the Banana data set \citep{IDA:MLJ:Ratsch+etal:2001}.
In this experiment, the class-prior was set at $\thetap=0.1$ 
and assumed to be known.

Figure \ref{fig:var_range_comb_para_samples} plots
the value of the variance of the empirical PNU-AUC risk 
divided by that of the PN-AUC risk, 
\begin{align*}
r:=\frac{\Var[\Rh_\PNU^\eta(\widehat{f}_\PN)]}
		{\Var[\Rh_\PN(\widehat{f}_\PN)]} ,
\end{align*}
as a function of the combination parameter $\eta$
under different numbers of positive and negative samples.
The results show that $r<1$ can be achieved by 
an appropriate choice of $\eta$, 
meaning that the variance of the empirical PNU-AUC risk can be
smaller than that of the PN-AUC risk.

\begin{figure}[t]
	\centering
 	\subfigure[$-1\leq\eta\leq1$]{%
\ifmlj
 		\includegraphics[clip, width=.5\columnwidth]
 		{all_banana_var_range_pi10.pdf}
\else 		
 		\includegraphics[clip, width=.5\columnwidth]
 		{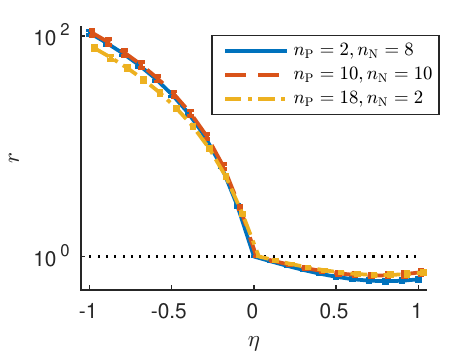}
\fi 		
		}%
 	\subfigure[$\eta>0$]{%
\ifmlj
 		\includegraphics[clip, width=.5\columnwidth]
 		{banana_var_range_pi10.pdf}
\else
 		\includegraphics[clip, width=.5\columnwidth]
 		{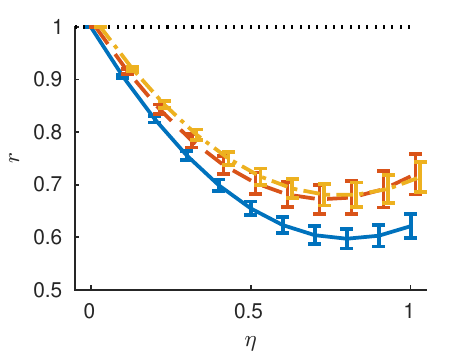}
\fi
		}
	\caption{
	Average with standard error of the ratio 
	between the variance of the empirical PNU risk and that of the PN risk, 
	$r=\Var[\Rh_\PNU^\eta(\fh_\PN)]/\Var[\Rh_\PN(\fh_\PN)]$,
	as a function of the combination parameter $\eta$  
	over $100$ trials on the Banana data set.
	The class-prior is $\thetap=0.1$
	and the number of positive and negative samples varies as
	$(\np,\nn)=(2,8)$, $(10,10)$, and $(18,2)$.
	Left: values of $r$ as a function of $\eta$.
	Right: values for $\eta>0$ are magnified.
	}	
	\label{fig:var_range_comb_para_samples}
\end{figure}

We then investigate how the class-prior affects the variance reduction.
In this experiment, 
the number of positive and negative samples for $\fh_\PN$ 
are $\np=10$ and $\nn=10$, respectively.
Figure~\ref{fig:var_range_comb_para_priors}
showed the values of $r$ as a function of the
combination parameter $\eta$ under different class-priors.
When the class-prior, $\thetap$, is $0.1$ and $0.2$, the variance 
can be reduced for $\eta>0$.
When the class-prior is $0.3$,
the range of the value of $\eta$ that yields 
variance reduction becomes smaller.
However, this may not be
that problematic in practice,
because AUC optimization is effective 
when two classes are highly imbalanced, 
i.e., the class-prior is far from $0.5$;
when the class-prior is close to $0.5$,
we may simply use the standard misclassification rate minimization approach.

\begin{figure}[t]
	\centering
 	\subfigure[$-1\leq\eta\leq1$]{%
\ifmlj 	
 		\includegraphics[clip, width=.5\columnwidth]
 		{all_banana_var_range_np10_nn10.pdf}
\else
 		\includegraphics[clip, width=.5\columnwidth]
 		{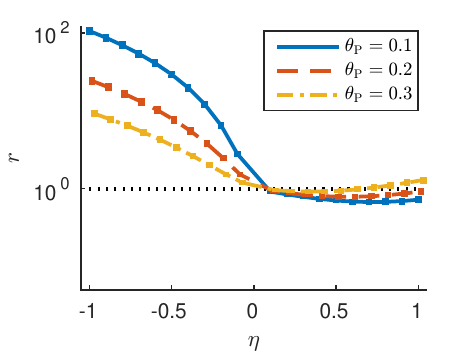}
\fi
		}%
 	\subfigure[$\eta>0$]{%
\ifmlj
 		\includegraphics[clip, width=.5\columnwidth]
 		{banana_var_range_np10_nn10.pdf}
\else
 		\includegraphics[clip, width=.5\columnwidth]
 		{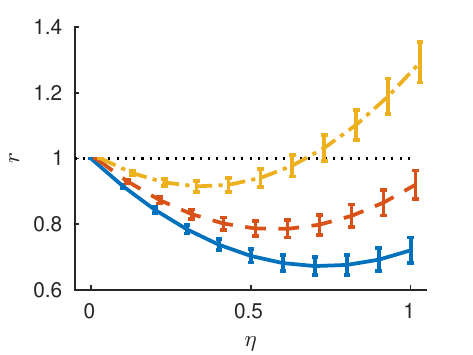}
\fi
		}
	\caption{
	Average with standard error of the ratio 
	between the variance of the PNU-AUC risk and that of the PN-AUC risk, 
	$r=\Var[\Rh_\PNU^\eta(\fh_\PN)]/\Var[\Rh_\PN(\fh_\PN)]$,
	as a function of the combination parameter $\eta$ 
	over $100$ trials
	on the Banana data set.
	A class-prior varies as $\thetap=0.1$, $0.2$, and $0.3$.
	Left: values of $r$ as a function of $\eta$.
	Right: values for $\eta>0$ are magnified.	
	When $\thetap=0.1$, $0.2$,
	the variance of the empirical PNU-AUC risk is smaller than 
	that of the PN risk for $\eta>0$. 
	}	
	\label{fig:var_range_comb_para_priors}
\end{figure}

\subsection{Benchmark Data Sets}
Next, we report the classification performance of the proposed PU-AUC 
and PNU-AUC optimization methods, respectively.
We used $15$ benchmark data sets from
the \emph{IDA Benchmark Repository} \citep{IDA:MLJ:Ratsch+etal:2001},
the \emph{Semi-Supervised Learning Book} \citep{book:Chapelle+etal:2006},
the \emph{LIBSVM} \citep{LibSVM:Chang+etal:2011},
and the \emph{UCI Machine Learning Repository} \citep{UCI:Lichman:2013}.
The detailed statistics of the data sets are summarized in
Appendix~\ref{app:stats_data}.

\subsubsection{AUC Optimization from Positive and Unlabeled Data}
We compared the proposed PU-AUC optimization method against 
the existing AUC optimization method based on the ranking SVM (PU-RSVM)
\citep{CIKM:Sundararajan+etal:2011}.
We trained a classifier with samples of size
$\np=100$ and $\nun=1000$ under the different class-priors
$\thetap=0.1$ and $0.2$.
For the PU-AUC optimization method, 
the squared loss function was used
and the class-prior was estimated by the distribution matching method
\citep{MLJ:duPlessis+etal:2017}.
The results of the estimated class-prior are summarized 
in Table~\ref{tab:bench-pu-pi-ret}.

\begin{table}[t]	
	\centering	
	\caption{Average and standard error of the 
	estimated class-prior over $50$ trials on benchmark data sets
	in PU learning setting.
	}		
	\label{tab:bench-pu-pi-ret}
  	\begin{tabular}{lrrr}
		\toprule 
		Data set 
		& \multicolumn{1}{c}{$d$}		
		& \multicolumn{1}{c}{$\thetap=0.1$}
		& \multicolumn{1}{c}{$\thetap=0.2$} \\
		\toprule				
	Banana
&$2$
& 0.15 (0.01)
& 0.24 (0.01)
	\\ 
	\cmidrule(lr){1-4} 
	skin\_nonskin
&$3$
& 0.10 (0.00)
& 0.20 (0.01)
	\\ 
	\cmidrule(lr){1-4} 
	cod-rna
&$8$
& 0.32 (0.01)
& 0.40 (0.02)
	\\ 
	\cmidrule(lr){1-4} 
	Magic
&$10$
& 0.34 (0.01)
& 0.39 (0.01)
	\\ 
	\cmidrule(lr){1-4} 
	Image
&$18$
& 0.13 (0.01)
& 0.24 (0.01)
	\\ 
	\cmidrule(lr){1-4} 
	Twonorm
&$20$
& 0.27 (0.00)
& 0.34 (0.00)
	\\ 
	\cmidrule(lr){1-4} 
	Waveform
&$21$
& 0.31 (0.00)
& 0.38 (0.01)
	\\ 
	\cmidrule(lr){1-4} 
	mushrooms
&$112$
& 0.19 (0.00)
& 0.28 (0.00)
	\\  
	\bottomrule
	\end{tabular}	
\end{table}	

Table~\ref{tab:bench-pu} lists the average with standard error of 
the AUC over $50$ trials, showing that
the proposed PU-AUC optimization method achieves better performance
than the existing method.
In particular, when $\thetap=0.2$, 
the difference between PU-RSVM and our method
becomes larger compared with the difference 
when $\thetap=0.1$.
Since PU-RSVM can be interpreted as regarding 
unlabeled data as negative, the bias caused by 
this becomes larger when $\thetap=0.2$. 
  
Figure~\ref{fig:bench-time-pu} summarizes
the average computation time over $50$ trials.
The computation time of the PU-AUC optimization method includes
both the class-prior estimation and the empirical risk minimization.
The results show that  
the PU-AUC optimization method requires almost twice
computation time as that of PU-RSVM,
but it would be acceptable in practice
to obtain better performance.

\begin{table}[t]
	\centering
	\caption{Average and standard error of the AUC 
	over $50$ trials on benchmark data sets.
    The boldface denotes the best and comparable methods in terms of 
    the average AUC according to the t-test at the significance level $5\%$.
    The last row shows the number of best/comparable cases of each method.    
	}		
	\label{tab:bench-pu}
	\begin{tabular}{lrrrrr}
		\toprule
		\multirow{2}{*}{Data set}
		& \multirow{2}{*}{$d$} 
		& \multicolumn{2}{c}{$\thetap=0.1$}
		& \multicolumn{2}{c}{$\thetap=0.2$} \\ 
		\cmidrule(lr){3-4}\cmidrule(lr){5-6}
		& 
 		& \multicolumn{1}{c}{PU-AUC} 		
		& \multicolumn{1}{c}{PU-RSVM}		
 		& \multicolumn{1}{c}{PU-AUC}
		& \multicolumn{1}{c}{PU-RSVM} \\ 		 
		\cmidrule(lr){1-6}
	Banana
&$2$
& $\mathbf{95.4}$ ($\mathbf{0.1}$) 
& $88.5$ ($0.2$) 
& $\mathbf{95.0}$ ($\mathbf{0.1}$) 
& $85.9$ ($0.2$) 
	\\ 
	\cmidrule(lr){1-6} 
	skin\_nonskin
&$3$
& $\mathbf{99.9}$ ($\mathbf{0.0}$) 
& $86.2$ ($0.3$) 
& $\mathbf{99.8}$ ($\mathbf{0.0}$) 
& $73.8$ ($0.5$) 
	\\ 
	\cmidrule(lr){1-6} 
	cod-rna
&$8$
& $\mathbf{98.1}$ ($\mathbf{0.1}$) 
& $62.8$ ($0.3$) 
& $\mathbf{97.6}$ ($\mathbf{0.1}$) 
& $59.4$ ($0.4$) 
	\\ 
	\cmidrule(lr){1-6} 
	Magic
&$10$
& $\mathbf{87.4}$ ($\mathbf{0.2}$) 
& $82.8$ ($0.2$) 
& $\mathbf{86.4}$ ($\mathbf{0.2}$) 
& $81.9$ ($0.1$) 
	\\ 
	\cmidrule(lr){1-6} 
	Image
&$18$
& $\mathbf{97.5}$ ($\mathbf{0.1}$) 
& $82.2$ ($0.2$) 
& $\mathbf{96.7}$ ($\mathbf{0.2}$) 
& $77.0$ ($0.4$) 
	\\ 
	\cmidrule(lr){1-6} 
	Twonorm
&$20$
& $\mathbf{99.6}$ ($\mathbf{0.0}$) 
& $85.5$ ($0.2$) 
& $\mathbf{99.5}$ ($\mathbf{0.0}$) 
& $79.2$ ($0.4$) 
	\\ 
	\cmidrule(lr){1-6} 
	Waveform
&$21$
& $\mathbf{96.6}$ ($\mathbf{0.1}$) 
& $73.9$ ($0.6$) 
& $\mathbf{96.3}$ ($\mathbf{0.1}$) 
& $59.5$ ($1.0$) 
	\\ 
	\cmidrule(lr){1-6} 
	mushrooms
&$112$
& $\mathbf{99.8}$ ($\mathbf{0.0}$) 
& $93.9$ ($0.3$) 
& $\mathbf{99.6}$ ($\mathbf{0.1}$) 
& $84.7$ ($0.3$) 
	\\ 
	\cmidrule(lr){1-6} 
  		\cmidrule{1-6}
  		\multicolumn{2}{c}{\#Best/Comp.~} 
   		& \multicolumn{1}{r}{$8$} 
   		& \multicolumn{1}{r}{$0$} 
   		& \multicolumn{1}{r}{$8$} 
   		& \multicolumn{1}{r}{$0$} \\
		\bottomrule	
	\end{tabular}		
\end{table}
\begin{figure}[t]
	\centering
\ifmlj	
	\includegraphics[clip,width=\columnwidth]
	{times_summary_pu.pdf}
\else
	\includegraphics[clip,width=\columnwidth]
	{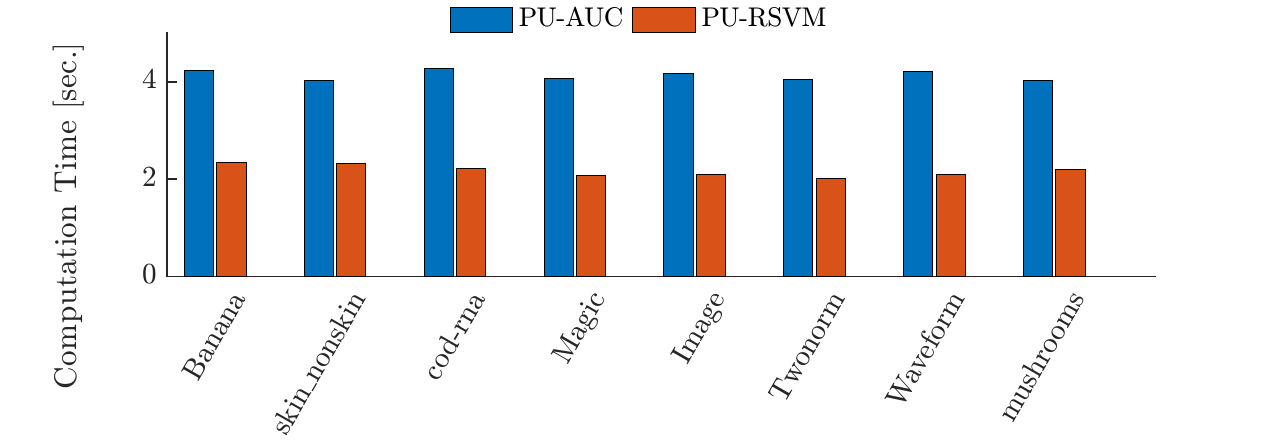}
\fi	
	\caption{
	Average computation time 
	on benchmark data sets when $\thetap=0.1$ over $50$ trials.
	The computation time of the PU-AUC optimization method 
	includes the class-prior estimation and 
	the empirical risk minimization.
	}\label{fig:bench-time-pu}
\end{figure}

\subsubsection{Semi-Supervised AUC Optimization}
\label{sec:exp-ssl}
Here, we compare the proposed PNU-AUC optimization method against 
existing AUC optimization approaches: 
the semi-supervised rankboost (SSRankboost)
\citep{SIGIR:Amini+etal:2008},\footnote{
We used the code available at
\url{http://ama.liglab.fr/~amini/SSRankBoost/}
} 
the semi-supervised AUC-optimized logistic sigmoid (sAUC-LS) 
\citep{ICDM:Fujino+Ueda:2016},\footnote{
This method is equivalent to OptAG without a generative model,
which only employs a discriminative model
with the entropy minimization principle. 
To eliminate the adverse effect of the wrongly chosen
generative model, 
we added this method for comparison.
}
and the optimum AUC with a generative model (OptAG)
\citep{ICDM:Fujino+Ueda:2016}.

We trained the classifier with samples of size
$\np=\thetap\cdot\nl$, $\nn=\nl-\np$, and $\nun=1000$, 
where $\nl$ is the number of labeled samples.
For the PNU-AUC optimization method, 
the squared loss function was used and
the candidates of the combination parameter $\eta$ were 
$\{-0.9, -0.8, \ldots, 0.9\}$.
For the class-prior estimation, we used 
the energy distance minimization method
\citep{IEICE:Kawakubo+etal:2016}.
The results of the estimated class-prior are summarized 
in Table~\ref{tab:bench-pi-ret}.

\begin{table}[t]
	\centering
	\caption{Average and standard error of the 
	estimated class-prior over $50$ trials on benchmark data sets
	in semi-supervised learning setting.
	}		
	\label{tab:bench-pi-ret}
  	\begin{tabular}{lrrr}
		\toprule 
		\multirow{1}{*}{Data set} 
		& \multirow{1}{*}{$\nl$}		
		& \multicolumn{1}{c}{$\thetap=0.1$}
		& \multicolumn{1}{c}{$\thetap=0.2$} \\
		\toprule				
	Banana
	& $50$
& 0.12 (0.01)
& 0.21 (0.02)
	\\ 
	($d=2$)
	& $100$
& 0.10 (0.01)
& 0.20 (0.01)
	\\ 
	\cmidrule(lr){1-4} 
	skin\_nonskin
	& $50$
& 0.11 (0.01)
& 0.21 (0.01)
	\\ 
	($d=3$)
	& $100$
& 0.10 (0.01)
& 0.20 (0.01)
	\\ 
	\cmidrule(lr){1-4} 
	cod-rna
	& $50$
& 0.12 (0.01)
& 0.22 (0.01)
	\\ 
	($d=8$)
	& $100$
& 0.12 (0.01)
& 0.21 (0.01)
	\\ 
	\cmidrule(lr){1-4} 
	Magic
	& $50$
& 0.09 (0.01)
& 0.17 (0.01)
	\\ 
	($d=10$)
	& $100$
& 0.07 (0.01)
& 0.20 (0.01)
	\\ 
	\cmidrule(lr){1-4} 
	Image
	& $50$
& 0.12 (0.01)
& 0.22 (0.01)
	\\ 
	($d=18$)
	& $100$
& 0.11 (0.01)
& 0.20 (0.01)
	\\ 
	\cmidrule(lr){1-4} 
	SUSY
	& $50$
& 0.10 (0.01)
& 0.20 (0.01)
	\\ 
	($d=18$)
	& $100$
& 0.10 (0.01)
& 0.19 (0.01)
	\\ 
	\cmidrule(lr){1-4} 
	Ringnorm
	& $50$
& 0.06 (0.00)
& 0.15 (0.00)
	\\ 
	($d=20$)
	& $100$
& 0.07 (0.00)
& 0.17 (0.00)
	\\ 
	\cmidrule(lr){1-4} 
	Twonorm
	& $50$
& 0.10 (0.00)
& 0.20 (0.00)
	\\ 
	($d=20$)
	& $100$
& 0.10 (0.00)
& 0.20 (0.00)
	\\ 
	\cmidrule(lr){1-4} 
	Waveform
	& $50$
& 0.11 (0.01)
& 0.20 (0.01)
	\\ 
	($d=21$)
	& $100$
& 0.09 (0.01)
& 0.19 (0.01)
	\\ 
	\cmidrule(lr){1-4} 
	covtype
	& $50$
& 0.09 (0.01)
& 0.20 (0.01)
	\\ 
	($d=54$)
	& $100$
& 0.09 (0.01)
& 0.18 (0.01)
	\\ 
	\cmidrule(lr){1-4} 
	phishing
	& $50$
& 0.10 (0.00)
& 0.20 (0.00)
	\\ 
	($d=68$)
	& $100$
& 0.10 (0.00)
& 0.20 (0.00)
	\\ 
	\cmidrule(lr){1-4} 
	a9a
	& $50$
& 0.10 (0.01)
& 0.20 (0.01)
	\\ 
	($d=83$)
	& $100$
& 0.10 (0.00)
& 0.21 (0.01)
	\\ 
	\cmidrule(lr){1-4} 
	mushrooms
	& $50$
& 0.10 (0.00)
& 0.20 (0.00)
	\\ 
	($d=112$)
	& $100$
& 0.10 (0.00)
& 0.20 (0.00)
	\\ 
	\cmidrule(lr){1-4} 
	USPS
	& $50$
& 0.10 (0.00)
& 0.20 (0.01)
	\\ 
	($d=241$)
	& $100$
& 0.09 (0.01)
& 0.19 (0.01)
	\\ 
	\cmidrule(lr){1-4} 
	w8a
	& $50$
& 0.10 (0.00)
& 0.19 (0.00)
	\\ 
	($d=300$)
	& $100$
& 0.09 (0.00)
& 0.20 (0.01)
	\\ 
		\bottomrule
	\end{tabular}			
\end{table}

For SSRankboost,
the discount factor and the number of neighbors 
were chosen from $\{10^{-3},10^{-2},10^{-1}\}$
and $\{2,3,\ldots,7\}$, respectively.
For sAUC-LS and OptAG,
the regularization parameter for the entropy regularizer
was chosen from $\{1, 10\}$.
Furthermore, as the generative model of OptAG,  
we adapted the Gaussian distribution for the data distribution
and the Gaussian and Gamma distributions for the prior of
the data distribution.
\footnote{
As the generative model, we used 
the Gaussian distributions for positive and negative classes:
\begin{align*}
p_g(\bxp; \bmup)&\propto \taup^\frac{d}{2}
	\exp\Big(-\frac{\taup}{2}\|\bxp-\bmup\|^2\Big) , \\
p_g(\bxp; \bmun)&\propto \taun^\frac{d}{2}
	\exp\Big(-\frac{\taun}{2}\|\bxn-\bmun\|^2\Big) ,
\end{align*}
where $\taup$ and $\taun$ denote the precisions
and $\bmup$ and $\bmun$ are the means.
As the prior of $\bmup$, $\bmun$, $\taup$, and $\taun$,  
we used the Gaussian and Gamma distributions:
\begin{align*}	
p(\bmup;\bmup^0)&\propto \taup^\frac{d}{2}
	\exp\Big(-\frac{\rhop^0\taup}{2}\|\bmup-\bmup^0\|^2\Big) , \\
p(\bmun;\bmun^0)&\propto \taun^\frac{d}{2}
	\exp\Big(-\frac{\rhon^0\taun}{2}\|\bmun-\bmun^0\|^2\Big) , \\
p(\taup; \ap^0,\bp^0)&\propto \taup^{\ap^0-1}\exp(-\bp^0\taup) , \\ 	 
p(\taun; \an^0,\bn^0)&\propto \taun^{\an^0-1}\exp(-\bn^0\taun) , \\ 	 
\end{align*}
where $\bmup^0$, $\bmu^0$, $\ap^0$, $\bp^0$, $\an^0$, $\bn^0$,
$\rhop^0$, and $\rhon^0$ are the hyperparameters.
}

Table~\ref{tab:bench} lists the average with standard error of 
the AUC over $50$ trials, showing that
the proposed PNU-AUC optimization method 
achieves better performance than or comparable performance to the existing
methods on many data sets. 
 
Figure~\ref{fig:bench-time} summarizes
the average computation time over $50$ trials.
The computation time of the PNU-AUC optimization method 
includes both the class-prior estimation and the empirical risk minimization.
The results show that even though  
our proposed method involves the class-prior estimation, 
the computation time is relatively faster than 
SSRankboost and much faster than sAUC-LS and OptAG.
The reason for longer computation time of sAUC-LS and OptAG is 
that their implementation is based on the logistic loss in which 
the number of operations for loss evaluation is $\cO(\np\nn+\np\nun+\nn\nun)$, 
unlike the PNU-AUC optimization method with the squared loss in which 
the number of operations for loss evaluation is $\cO(\np+\nn+\nun)$ 
(cf.~the discussion about the computational complexity in
Section~\ref{sec:impl}).

\begin{table}[t]
	\centering
	\caption{Average and standard error of the AUC 
	over $50$ trials on benchmark data sets.
    The boldface denotes the best and comparable methods in terms of 
    the average AUC according to the t-test at the significance level $5\%$.
    The last row shows the number of best/comparable cases of each method.
    SSRboost is an abbreviation for SSRankboost.    
	}		
	\label{tab:bench}
 	\resizebox{\columnwidth}{!}{%
   	\begin{tabular}{@{}l@{}r@{\;}c@{}c@{}c@{}c@{}c@{}c@{}c@{}c}
		\toprule 
		\multirow{2}{*}{Data set} 
		& \multirow{2}{*}{$\nl$}		
		& \multicolumn{4}{c}{$\thetap=0.1$}
		& \multicolumn{4}{c}{$\thetap=0.2$} \\
		\cmidrule(lr){3-6}\cmidrule(lr){7-10}
		& 		
 		& \multicolumn{1}{c}{PNU-AUC}
 		& \multicolumn{1}{c}{SSRboost} 
  		& \multicolumn{1}{c}{sAUC-LS} 
 		& \multicolumn{1}{c}{OptAG} 
 		& \multicolumn{1}{c}{PNU-AUC}
 		& \multicolumn{1}{c}{SSRboost}
  		& \multicolumn{1}{c}{sAUC-LS}  
 		& \multicolumn{1}{c}{OptAG} \\ 
		\toprule				
	Banana
	& $50$
& $\mathbf{84.6}$ ($\mathbf{1.3}$) 
& $61.6$ ($1.0$) 
& $\mathbf{82.0}$ ($\mathbf{1.3}$) 
& $\mathbf{84.0}$ ($\mathbf{1.3}$) 
& $\mathbf{86.7}$ ($\mathbf{0.9}$) 
& $66.1$ ($0.9$) 
& $83.3$ ($1.5$) 
& $\mathbf{86.9}$ ($\mathbf{0.7}$) 
	\\ 
	($d=2$)
	& $100$
& $\mathbf{88.4}$ ($\mathbf{0.7}$) 
& $66.2$ ($0.7$) 
& $84.0$ ($1.2$) 
& $\mathbf{88.7}$ ($\mathbf{0.5}$) 
& $\mathbf{91.2}$ ($\mathbf{0.4}$) 
& $69.1$ ($0.5$) 
& $88.5$ ($0.6$) 
& $89.9$ ($0.6$) 
	\\ 
	\cmidrule(lr){1-10} 
	skin\_nonskin
	& $50$
& $\mathbf{96.5}$ ($\mathbf{0.9}$) 
& $92.2$ ($0.8$) 
& $96.4$ ($0.6$) 
& $\mathbf{97.8}$ ($\mathbf{0.6}$) 
& $\mathbf{98.1}$ ($\mathbf{0.6}$) 
& $93.7$ ($0.6$) 
& $\mathbf{98.8}$ ($\mathbf{0.2}$) 
& $\mathbf{99.1}$ ($\mathbf{0.2}$) 
	\\ 
	($d=3$)
	& $100$
& $98.8$ ($0.3$) 
& $94.5$ ($0.4$) 
& $99.0$ ($0.1$) 
& $\mathbf{99.5}$ ($\mathbf{0.0}$) 
& $\mathbf{99.1}$ ($\mathbf{0.3}$) 
& $95.6$ ($0.2$) 
& $\mathbf{99.2}$ ($\mathbf{0.3}$) 
& $\mathbf{99.3}$ ($\mathbf{0.2}$) 
	\\ 
	\cmidrule(lr){1-10} 
	cod-rna
	& $50$
& $\mathbf{86.7}$ ($\mathbf{1.4}$) 
& $79.4$ ($1.0$) 
& $61.6$ ($1.5$) 
& $63.3$ ($1.4$) 
& $\mathbf{91.2}$ ($\mathbf{1.0}$) 
& $86.8$ ($0.6$) 
& $66.0$ ($1.5$) 
& $68.7$ ($1.3$) 
	\\ 
	($d=8$)
	& $100$
& $\mathbf{93.4}$ ($\mathbf{0.6}$) 
& $88.7$ ($0.5$) 
& $67.3$ ($1.5$) 
& $67.7$ ($1.3$) 
& $\mathbf{95.6}$ ($\mathbf{0.5}$) 
& $91.7$ ($0.3$) 
& $74.1$ ($1.5$) 
& $75.9$ ($0.9$) 
	\\ 
	\cmidrule(lr){1-10} 
	Magic
	& $50$
& $\mathbf{77.0}$ ($\mathbf{1.1}$) 
& $74.8$ ($0.8$) 
& $73.7$ ($1.6$) 
& $\mathbf{77.4}$ ($\mathbf{0.7}$) 
& $\mathbf{77.2}$ ($\mathbf{1.3}$) 
& $\mathbf{78.2}$ ($\mathbf{0.5}$) 
& $76.5$ ($0.7$) 
& $75.4$ ($1.1$) 
	\\ 
	($d=10$)
	& $100$
& $\mathbf{79.5}$ ($\mathbf{0.5}$) 
& $\mathbf{79.3}$ ($\mathbf{0.6}$) 
& $74.5$ ($1.6$) 
& $77.1$ ($1.0$) 
& $\mathbf{81.5}$ ($\mathbf{0.4}$) 
& $\mathbf{81.2}$ ($\mathbf{0.4}$) 
& $75.8$ ($1.0$) 
& $78.3$ ($0.5$) 
	\\ 
	\cmidrule(lr){1-10} 
	Image
	& $50$
& $\mathbf{81.6}$ ($\mathbf{1.7}$) 
& $70.0$ ($1.1$) 
& $76.5$ ($1.6$) 
& $\mathbf{80.7}$ ($\mathbf{1.6}$) 
& $\mathbf{86.1}$ ($\mathbf{0.8}$) 
& $78.8$ ($0.6$) 
& $81.5$ ($1.0$) 
& $82.9$ ($1.1$) 
	\\ 
	($d=18$)
	& $100$
& $\mathbf{88.3}$ ($\mathbf{0.6}$) 
& $80.9$ ($0.7$) 
& $83.4$ ($1.2$) 
& $85.2$ ($1.0$) 
& $\mathbf{92.1}$ ($\mathbf{0.4}$) 
& $87.1$ ($0.5$) 
& $86.5$ ($0.6$) 
& $87.8$ ($0.4$) 
	\\ 
	\cmidrule(lr){1-10} 
	SUSY
	& $50$
& $63.0$ ($1.1$) 
& $\mathbf{67.7}$ ($\mathbf{1.3}$) 
& $56.2$ ($0.9$) 
& $56.2$ ($1.0$) 
& $64.7$ ($0.8$) 
& $\mathbf{70.9}$ ($\mathbf{0.8}$) 
& $55.6$ ($0.9$) 
& $57.7$ ($0.7$) 
	\\ 
	($d=18$)
	& $100$
& $65.4$ ($1.1$) 
& $\mathbf{73.4}$ ($\mathbf{0.5}$) 
& $59.0$ ($0.9$) 
& $59.0$ ($0.8$) 
& $69.4$ ($1.0$) 
& $\mathbf{76.2}$ ($\mathbf{0.4}$) 
& $58.9$ ($0.7$) 
& $58.3$ ($0.6$) 
	\\ 
	\cmidrule(lr){1-10} 
	Ringnorm
	& $50$
& $\mathbf{98.3}$ ($\mathbf{0.4}$) 
& $79.9$ ($0.7$) 
& $\mathbf{98.7}$ ($\mathbf{0.5}$) 
& $\mathbf{99.0}$ ($\mathbf{0.6}$) 
& $98.4$ ($0.4$) 
& $86.6$ ($0.4$) 
& $\mathbf{99.5}$ ($\mathbf{0.3}$) 
& $\mathbf{99.1}$ ($\mathbf{0.4}$) 
	\\ 
	($d=20$)
	& $100$
& $98.8$ ($0.3$) 
& $88.0$ ($0.4$) 
& $\mathbf{99.8}$ ($\mathbf{0.0}$) 
& $\mathbf{99.8}$ ($\mathbf{0.0}$) 
& $\mathbf{99.4}$ ($\mathbf{0.2}$) 
& $90.8$ ($0.3$) 
& $\mathbf{99.5}$ ($\mathbf{0.4}$) 
& $\mathbf{99.6}$ ($\mathbf{0.2}$) 
	\\ 
	\cmidrule(lr){1-10} 
	Twonorm
	& $50$
& $\mathbf{96.9}$ ($\mathbf{0.6}$) 
& $90.3$ ($0.4$) 
& $94.4$ ($1.1$) 
& $\mathbf{96.1}$ ($\mathbf{0.3}$) 
& $\mathbf{97.5}$ ($\mathbf{0.5}$) 
& $93.2$ ($0.3$) 
& $\mathbf{97.7}$ ($\mathbf{0.2}$) 
& $\mathbf{97.4}$ ($\mathbf{0.2}$) 
	\\ 
	($d=20$)
	& $100$
& $\mathbf{98.6}$ ($\mathbf{0.1}$) 
& $94.7$ ($0.2$) 
& $96.6$ ($0.2$) 
& $96.5$ ($0.2$) 
& $\mathbf{99.0}$ ($\mathbf{0.1}$) 
& $96.8$ ($0.1$) 
& $98.3$ ($0.1$) 
& $98.0$ ($0.2$) 
	\\ 
	\cmidrule(lr){1-10} 
	Waveform
	& $50$
& $\mathbf{86.7}$ ($\mathbf{1.3}$) 
& $\mathbf{88.9}$ ($\mathbf{0.4}$) 
& $85.5$ ($1.5$) 
& $85.8$ ($0.8$) 
& $\mathbf{92.0}$ ($\mathbf{0.6}$) 
& $\mathbf{91.3}$ ($\mathbf{0.2}$) 
& $88.6$ ($0.8$) 
& $89.4$ ($0.6$) 
	\\ 
	($d=21$)
	& $100$
& $\mathbf{92.8}$ ($\mathbf{0.4}$) 
& $91.8$ ($0.2$) 
& $87.8$ ($1.1$) 
& $87.7$ ($1.2$) 
& $\mathbf{94.7}$ ($\mathbf{0.2}$) 
& $93.1$ ($0.1$) 
& $90.1$ ($0.4$) 
& $89.7$ ($0.5$) 
	\\ 
	\cmidrule(lr){1-10} 
	covtype
	& $50$
& $57.8$ ($1.3$) 
& $\mathbf{63.1}$ ($\mathbf{1.0}$) 
& $55.7$ ($0.9$) 
& $58.9$ ($1.1$) 
& $60.3$ ($1.0$) 
& $\mathbf{65.6}$ ($\mathbf{0.8}$) 
& $56.4$ ($0.9$) 
& $57.8$ ($0.9$) 
	\\ 
	($d=54$)
	& $100$
& $60.7$ ($1.1$) 
& $\mathbf{66.7}$ ($\mathbf{0.8}$) 
& $59.1$ ($0.9$) 
& $60.3$ ($0.9$) 
& $64.2$ ($0.6$) 
& $\mathbf{70.6}$ ($\mathbf{0.6}$) 
& $57.7$ ($0.9$) 
& $60.6$ ($0.7$) 
	\\ 
	\cmidrule(lr){1-10} 
	phishing
	& $50$
& $89.8$ ($1.1$) 
& $\mathbf{91.9}$ ($\mathbf{0.6}$) 
& $71.1$ ($1.6$) 
& $74.0$ ($0.9$) 
& $91.8$ ($0.7$) 
& $\mathbf{94.2}$ ($\mathbf{0.3}$) 
& $74.7$ ($1.5$) 
& $77.1$ ($1.2$) 
	\\ 
	($d=68$)
	& $100$
& $\mathbf{94.6}$ ($\mathbf{0.2}$) 
& $\mathbf{94.8}$ ($\mathbf{0.2}$) 
& $76.3$ ($1.3$) 
& $76.9$ ($1.2$) 
& $94.9$ ($0.4$) 
& $\mathbf{96.3}$ ($\mathbf{0.1}$) 
& $80.5$ ($1.3$) 
& $82.0$ ($1.0$) 
	\\ 
	\cmidrule(lr){1-10} 
	a9a
	& $50$
& $69.4$ ($1.7$) 
& $\mathbf{75.1}$ ($\mathbf{0.9}$) 
& $\mathbf{75.3}$ ($\mathbf{1.4}$) 
& $\mathbf{73.8}$ ($\mathbf{1.4}$) 
& $\mathbf{78.2}$ ($\mathbf{1.0}$) 
& $\mathbf{79.3}$ ($\mathbf{0.5}$) 
& $\mathbf{78.3}$ ($\mathbf{1.1}$) 
& $\mathbf{79.2}$ ($\mathbf{0.7}$) 
	\\ 
	($d=83$)
	& $100$
& $77.9$ ($1.1$) 
& $\mathbf{80.1}$ ($\mathbf{0.5}$) 
& $\mathbf{80.3}$ ($\mathbf{0.5}$) 
& $\mathbf{79.5}$ ($\mathbf{0.8}$) 
& $82.0$ ($0.5$) 
& $\mathbf{83.2}$ ($\mathbf{0.3}$) 
& $81.5$ ($0.5$) 
& $81.5$ ($0.7$) 
	\\ 
	\cmidrule(lr){1-10} 
	mushrooms
	& $50$
& $96.0$ ($0.7$) 
& $96.0$ ($0.6$) 
& $96.6$ ($0.5$) 
& $\mathbf{97.8}$ ($\mathbf{0.3}$) 
& $96.4$ ($0.7$) 
& $\mathbf{98.0}$ ($\mathbf{0.2}$) 
& $\mathbf{97.2}$ ($\mathbf{0.5}$) 
& $\mathbf{97.6}$ ($\mathbf{0.8}$) 
	\\ 
	($d=112$)
	& $100$
& $\mathbf{98.1}$ ($\mathbf{0.5}$) 
& $98.2$ ($0.1$) 
& $97.9$ ($0.3$) 
& $\mathbf{98.9}$ ($\mathbf{0.2}$) 
& $\mathbf{98.7}$ ($\mathbf{0.2}$) 
& $\mathbf{98.7}$ ($\mathbf{0.1}$) 
& $98.2$ ($0.4$) 
& $\mathbf{99.1}$ ($\mathbf{0.4}$) 
	\\ 
	\cmidrule(lr){1-10} 
	USPS
	& $50$
& $\mathbf{85.6}$ ($\mathbf{1.0}$) 
& $73.9$ ($0.9$) 
& $79.5$ ($1.4$) 
& $82.7$ ($0.8$) 
& $\mathbf{88.2}$ ($\mathbf{0.7}$) 
& $80.5$ ($0.7$) 
& $81.2$ ($0.9$) 
& $85.1$ ($1.2$) 
	\\ 
	($d=241$)
	& $100$
& $\mathbf{90.3}$ ($\mathbf{0.5}$) 
& $79.4$ ($0.7$) 
& $84.2$ ($1.0$) 
& $86.0$ ($1.0$) 
& $\mathbf{93.7}$ ($\mathbf{0.5}$) 
& $85.2$ ($0.5$) 
& $83.3$ ($0.9$) 
& $89.7$ ($0.3$) 
	\\ 
	\cmidrule(lr){1-10} 
	w8a
	& $50$
& $\mathbf{69.6}$ ($\mathbf{1.1}$) 
& $\mathbf{70.9}$ ($\mathbf{0.9}$) 
& $52.5$ ($0.9$) 
& $52.9$ ($1.1$) 
& $\mathbf{79.2}$ ($\mathbf{0.9}$) 
& $75.4$ ($0.9$) 
& $52.6$ ($0.9$) 
& $55.2$ ($0.9$) 
	\\ 
	($d=300$)
	& $100$
& $\mathbf{79.9}$ ($\mathbf{1.1}$) 
& $77.1$ ($0.7$) 
& $53.1$ ($1.0$) 
& $53.9$ ($1.1$) 
& $\mathbf{85.6}$ ($\mathbf{0.6}$) 
& $81.0$ ($0.5$) 
& $55.7$ ($0.8$) 
& $56.2$ ($0.9$) 
\\
   		\cmidrule{1-10}		 	
   		\multicolumn{2}{c}{\#Best/Comp.~} 
   		& \multicolumn{1}{r}{$20$} 
   		& \multicolumn{1}{r}{$11$} 
   		& \multicolumn{1}{r}{$5$} 
   		& \multicolumn{1}{r}{$13$} 
   		& \multicolumn{1}{r}{$21$} 
   		& \multicolumn{1}{r}{$13$} 
   		& \multicolumn{1}{r}{$7$}
   		& \multicolumn{1}{r}{$9$} \\  		
		\bottomrule
	\end{tabular}		
 	}
\end{table}
\begin{figure}[t]
	\centering
\ifmlj	
	\includegraphics[clip,width=\columnwidth]
	{times_summary.pdf}
\else
	\includegraphics[clip,width=\columnwidth]
	{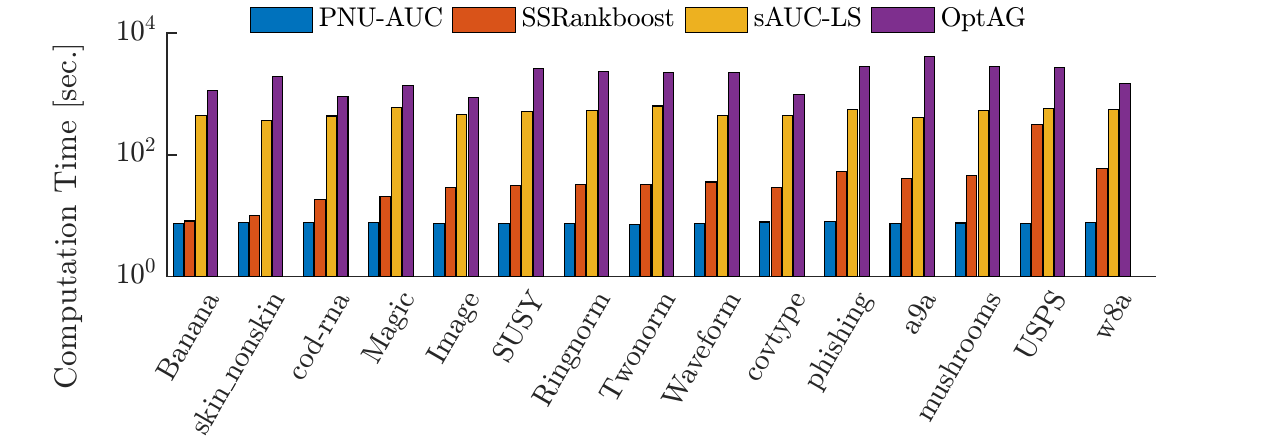}
\fi	
	\caption{
	Average computation time of each method 
	on benchmark data sets when $\nl=100$
	and $\thetap=0.1$ over $50$ trials.
	}\label{fig:bench-time}
\end{figure}

\subsection{Text Classification}
\label{sec:exp-text}
Next, we apply our proposed PNU-AUC optimization method to
a text classification task.
We used the \emph{Reuters Corpus Volume I} data set
\citep{JMLR:Lewis+etal:2004}, 
the \emph{Amazon Review} data set \citep{ICML:Dredze+etal:2008},
and the \emph{$20$ Newsgroups} data set \citep{ICML:Lang:1995}.
More specifically, we used the data set processed for 
a binary classification task: the rcv$1$, amazon$2$, and news$20$ data sets.
The rcv$1$ and news$20$ data sets are available at
the website of LIBSVM \citep{LibSVM:Chang+etal:2011},
and the amazon$2$ is designed by ourselves, which 
consists of the product reviews of books and
music from the \emph{Amazon$\mathit{7}$} data set
\citep{MLJ:Blondel+etal:2013}.
The dimension of a feature vector of the rcv$1$ data set is $47,236$, 
that of the amazon$2$ data set is $262,144$, 
and that of the news$20$ data set is $1,355,191$.

We trained a classifier with samples of size $\np=20$, $\nn=80$,
and $\nun=10,000$.
The true class-prior was set at $\thetap=0.2$ and estimated by
the method based on energy distance minimization
\citep{IEICE:Kawakubo+etal:2016}.
For the generative model of OptAG, we employed naive Bayes (NB) 
multinomial models 
and a Dirichlet prior for the prior distribution of the NB model
as described in \citet{ICDM:Fujino+Ueda:2016}. 

Table~\ref{tab:text} lists the average with standard error of 
the AUC over $20$ trials, showing that the proposed method outperforms
the existing methods. 
Figure~\ref{fig:text-time} summarizes the average computation time of 
each method.
These results show that the proposed method achieves better performance 
with short computation time.

\begin{table}[t]
	\centering
	\caption{Average with standard error of the AUC 
	over $20$ trials on the text classification data sets.
    The boldface denotes the best and comparable methods in terms of 
    the average AUC according to the t-test at the significance level $5\%$.
	}		
	\label{tab:text}
  	\begin{tabular}{lrrrrrr}
		\toprule 		
		Data set 
		& \multicolumn{1}{c}{$d$}		
		& \multicolumn{1}{c}{$\widehat{\theta}_\mathrm{P}$}		
 		& \multicolumn{1}{c}{PNU-AUC}
 		& \multicolumn{1}{c}{SSRankboost} 
  		& \multicolumn{1}{c}{sAUC-LS} 
 		& \multicolumn{1}{c}{OptAG} \\
 		\toprule				
	rcv$1$
& $47236$
& $0.21$ ($0.00$)
& $\mathbf{91.1}$ ($\mathbf{0.7}$) 
& $76.6$ ($0.7$) 
& $79.6$ ($2.3$) 
& $79.4$ ($2.0$) 
	\\ 
	amazon$2$
& $262144$
& $0.21$ ($0.01$)
& $\mathbf{87.2}$ ($\mathbf{1.3}$) 
& $68.7$ ($1.8$) 
& $55.9$ ($0.4$) 
& $56.8$ ($0.5$) 
	\\ 
	news$20$
& $1355191$
& $0.20$ ($0.00$)
& $\mathbf{79.8}$ ($\mathbf{1.1}$) 
& $74.2$ ($2.0$) 
& $66.4$ ($1.6$) 
& $71.5$ ($1.2$) 
	\\ 
		\bottomrule
	\end{tabular}		
\end{table}
\begin{figure}[t]
	\centering
\ifmlj	
	\includegraphics[clip,width=.6\columnwidth]
	{times_summary_text.pdf}
\else
	\includegraphics[clip,width=.6\columnwidth]
	{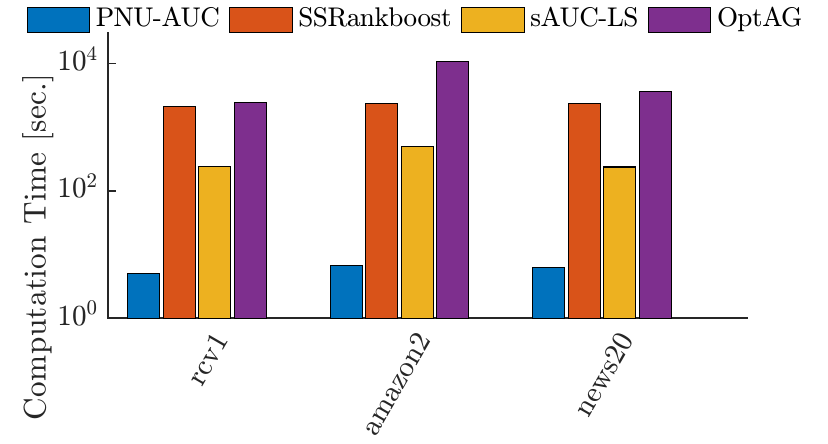}
\fi	
	\caption{
	Average computation time of each method 
	on the text classification data sets.
	}\label{fig:text-time}
\end{figure}

\subsection{Sensitivity Analysis}
Here, we investigate the effect of the estimation accuracy
of the class-prior for the PNU-AUC optimization method.
Specifically, 
we added noise $\rho\in\{-0.09,-0.08,\ldots,0.09\}$ to 
the true class-prior $\thetap$ and used 
$\widehat{\theta}_\mathrm{P}=\thetap+\rho$
as the estimated class-prior for the PNU-AUC optimization method.
Under the different values of the class-prior $\thetap=0.1$,
$0.2$, and $0.3$,
we trained a classifier with samples of size 
$\np=\thetap\cdot50$, $\nn=\thetan\cdot50$, and $\nun=1000$.

Figure~\ref{fig:sensitivity} summarizes
the average with standard error of the AUC as a function of the noise.
The plots show that when $\thetap=0.2$ and $0.3$,
the performance of the PNU-AUC optimization method is  
stable even when the estimated class-prior has some noise.
On the other hand, when $\thetap=0.1$, as the noise is close to $\rho=-0.09$,  
the performance largely decreases. 
Since the true class-prior is small,
it is sensitive to the negative bias. 
In particular, when $\rho=-0.09$, 
the gap between the estimated and true class-priors
is larger than other values.
For instance, when $\rho=-0.09$ and $\thetap=0.2$, 
$\thetap/\widehat{\theta}_\mathrm{P}\approx1.8$, 
but when $\rho=-0.09$ and $\thetap=0.1$, 
$\thetap/\widehat{\theta}_\mathrm{P}\approx10$. 
In contrast, the positive bias does not heavily affect 
the performance even when $\thetap=0.1$.

\begin{figure}[t]
	\centering
 	\subfigure[Banana ($d=2$)]{%
\ifmlj 	
 		\includegraphics[clip, width=.33\columnwidth]
 		{banana_nl50_nu1000_sensitivity.pdf}
\else
 		\includegraphics[clip, width=.33\columnwidth]
 		{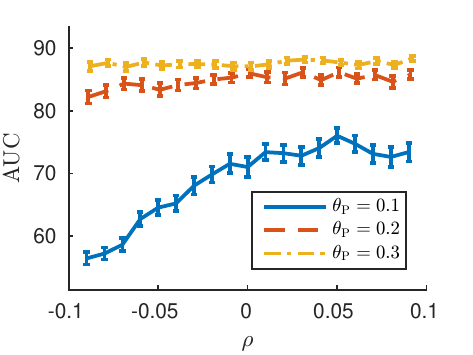}
\fi
		}%
 	\subfigure[cod-rna ($d=8$)]{%
\ifmlj
 		\includegraphics[clip, width=.33\columnwidth]
 		{cod-rna_nl50_nu1000_sensitivity.pdf}
\else
 		\includegraphics[clip, width=.33\columnwidth]
 		{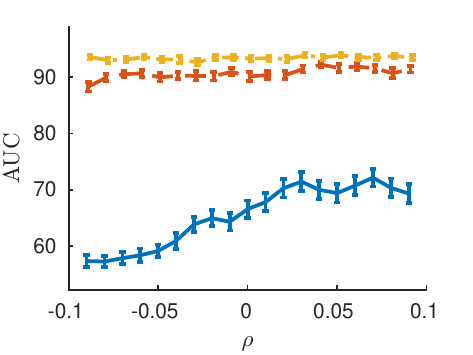}
\fi 		
		}%
 	\subfigure[w8a ($d=300$)]{%
\ifmlj
 		\includegraphics[clip, width=.33\columnwidth]
 		{w8a_nl50_nu1000_sensitivity.pdf}
\else
 		\includegraphics[clip, width=.33\columnwidth]
 		{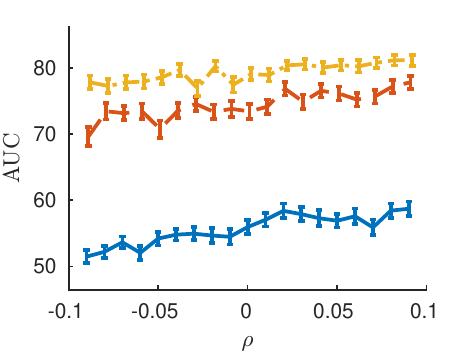}
\fi
		}
	\caption{
	Average with standard error of the AUC as a function of 
	the noise $\rho$ over $100$ trials.
	The PNU-AUC optimization method used the noisy class-prior 
	$\widehat{\theta}_\mathrm{P}=\thetap+\rho$
	in training.
	The plots show that 
	when $\thetap=0.2$ and $0.3$,
	the performance of the PNU-AUC optimization method 
	is stable even when the estimated class-prior has some noise.
	However, when $\thetap=0.1$, as the noise is close to $\rho=-0.09$, 
	the performance largely decreases.
	Since the true class-prior is small,
	it is sensitive to the negative bias. 
	}	
	\label{fig:sensitivity}
\end{figure}

\subsection{Scalability}
Finally, we report the scalability of our proposed PNU-AUC optimization method.
Specifically, we evaluated the AUC and computation time
while increasing the number of unlabeled samples.
We picked two large data sets: the SUSY and amazon$2$ data sets.
The number of positive and negative samples were $\np=40$ and $\nn=160$,
respectively.

Figure~\ref{fig:scalability} summarizes the average with standard error 
of the AUC and computation time as a function of the number of 
unlabeled samples.
The AUC on the SUSY data set slightly increased at $\nun=1,000,000$,
but the improvement on the amazon$2$ data set was not noticeable
or the performance decreased slightly.
In this experiment, the increase of the size of unlabeled data did not
significantly improve the performance of the classifier, 
but it did not affect adversely,
i.e., it did not cause significant performance degeneration.

The result of computation time  shows that 
the proposed PNU-AUC optimization method 
can handle approximately $1,000,000$ samples within reasonable 
computation time in this experiment.
The longer computation time on the SUSY data set before $\nun=10,000$ 
is because we need to choose one additional hyperparameter, 
i.e., the bandwidth of the Gaussian kernel basis function,
compared with the linear kernel basis function.
However, the effect gradually decreases;
after $\nun=10,000$, the matrix multiplication of the high dimensional matrix
on the amazon$2$ data set ($d=262,144$) requires more computation time 
than the SUSY data set ($d=18$).

\begin{figure}[t]
	\centering
 	\subfigure[AUC]{%
\ifmlj 	
 		\includegraphics[clip, width=.5\columnwidth]
 		{scalability_auc.pdf}
\else
 		\includegraphics[clip, width=.5\columnwidth]
 		{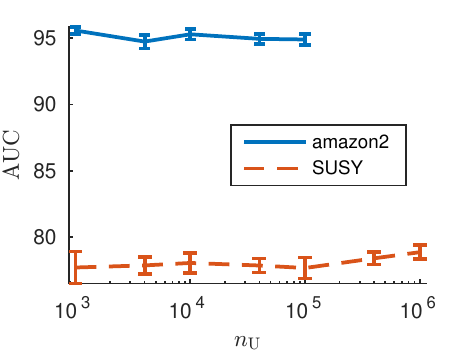}
\fi
		}%
 	\subfigure[Computation time]{%
\ifmlj
 		\includegraphics[clip, width=.5\columnwidth]
 		{scalability_time.pdf}
\else
 		\includegraphics[clip, width=.5\columnwidth]
 		{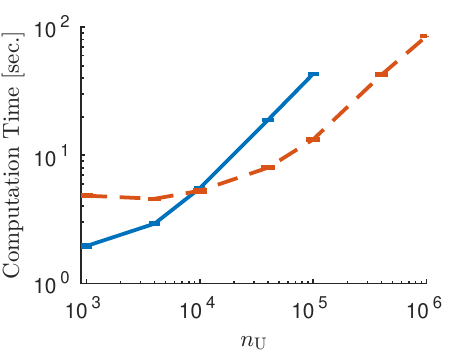}
\fi 		
		}%
	\caption{
	Average with standard error of the AUC as a function of 
	the number of unlabeled data $\nun$ over $20$ trials.
	}	
	\label{fig:scalability}
\end{figure}

\section{Conclusions}
In this paper, we proposed a novel AUC optimization 
method from positive and unlabeled data
and extend it to a novel semi-supervised AUC optimization method.
Unlike the existing approach,
our approach does not rely on strong distributional assumptions
on the data distributions such as the cluster and the entropy 
minimization principle.
Without the distributional assumptions,
we theoretically derived the generalization error bounds 
of our PU and semi-supervised AUC optimization methods.
Moreover, for our semi-supervised AUC optimization method, 
we showed that the variance of the empirical risk 
can be smaller than that of the supervised counterpart.
Through numerical experiments, 
we demonstrated the practical usefulness of the proposed 
PU and semi-supervised AUC optimization methods.

\subsection*{Acknowledgements}
TS was supported by KAKNEHI $15$J$09111$.
GN was supported by the JST CREST program and Microsoft Research Asia.
MS was supported by JST CREST JPMJCR$1403$.
We thank Han Bao for his comments.


\appendix

\section{PU-AUC Risk Estimator}
\label{app:emp_pu_risk}
In this section, we discuss the way of estimating the proposed PU-AUC risk.
Recall that the PU-AUC risk in Eq.~\eqref{eq:pu-risk} is defined as
\begin{align*}
R_\PU(f)=\frac{1}{\thetan}\Ep[\Eu[\ell(f(\bxp,\bxu))]]
	-\frac{\thetap}{\thetan}\Ep[\Epb[\ell(f(\bxp,\bbarxp))]] .
\end{align*}
If one additional set of positive samples $\{\bbarxp_i\}^\np_{i=1}$ 
is available, 
we obtain the unbiased PU-AUC risk estimator by 
\begin{align*}
\Rh_\PU(f)=\frac{1}{\thetan\np\nun}\sum^\np_{i=1}\sum^\nun_{k=1}
	\ell(f(\bxp_i,\bxu_k)) 
	- \frac{\thetap}{\thetan\np^2}\sum^\np_{i=1}\sum^\np_{i'=1}
	\ell(f(\bxp_i,\bbarxp_{i'})) .
\end{align*} 
We used this estimator in our theoretical analyses
because learning is not involved.
However, obtaining one additional set of samples is not always possible in practice.
Thus, instead of the above risk estimator, 
we use the following risk estimator in our implementation: 
\begin{align*}
\Rh_\PU(f)=\frac{1}{\thetan\np\nun}\sum^\np_{i=1}\sum^\nun_{k=1}
	\ell(f(\bxp_i,\bxu_k)) 
	- \frac{\thetap}{\thetan}
	\Bigg(\frac{1}{\np(\np-1)}\sum^\np_{i=1}\sum^\np_{i'=1}
	\ell(f(\bxp_i,\bxp_{i'})) - \frac{\ell(0)}{\np-1}\Bigg).
\end{align*}
This estimator is also unbiased.
To show unbiasedness of this estimator, 
let us rewrite the second term of the PU-AUC risk without coefficient 
in Eq.~\eqref{eq:pu-risk} as
\begin{align*}
\Ep[\Epb[\ell(f(\bxp,\bbarxp))]] = 
	\E\nolimits_{\bxp,\bbarxp}[\ell(f(\bxp,\bbarxp))] .
\end{align*}
The unbiased estimator can be expressed as
\begin{align*}
\frac{1}{\np(\np-1)}\sum^\np_{i=1}\sum^\np_{i'=1}
	\ell(f(\bxp_i,\bxp_{i'})) - \frac{\ell(0)}{\np-1} ,
\end{align*}
because the expectation of the above estimator
can be computed as follows: 
\begin{align*}
&\E\nolimits_{\bxp_1, \ldots, \bxp_\np}
	\Big[\frac{1}{\np(\np-1)}\sum^\np_{i=1}\sum^\np_{i'=1}
	\ell(f(\bxp_i,\bxp_{i'}))-\frac{\ell(0)}{\np-1}\Big] \\
&=\E\nolimits_{\bxp_1, \ldots, \bxp_\np}
	\Big[\frac{1}{\np(\np-1)}
	\Big(\sum^\np_{i=1}\ell(f(\bxp_i,\bxp_i))
	+\sum^\np_{i=1}\sum^\np_{i'\neq i}\ell(f(\bxp_i,\bxp_{i'}))\Big)
	-\frac{\ell(0)}{\np-1}\Big] \\
&=\E\nolimits_{\bxp_1, \ldots, \bxp_\np}
	\Big[\frac{1}{\np(\np-1)}\Big(\sum^\np_{i=1}\ell(0)
	+\sum^\np_{i=1}\sum^\np_{i'\neq i}\ell(f(\bxp_i,\bxp_{i'}))\Big)
	-\frac{\ell(0)}{\np-1}\Big] \\			
&=\frac{1}{\np(\np-1)}\sum^\np_{i=1}\sum^\np_{i'\neq i}
	\E\nolimits_{\bxp_i,\bxp_{i'}}	
	\Big[\ell(f(\bxp_i,\bxp_{i'}))\Big] \\
&=\frac{1}{\np(\np-1)}\sum^\np_{i=1}\sum^\np_{i'\neq i}
	\E\nolimits_{\bxp,\bbarxp}	
	\Big[\ell(f(\bxp,\bbarxp))\Big] \\	
&=\E\nolimits_{\bxp,\bbarxp}\Big[\ell(f(\bxp,\bbarxp))\Big] ,
\end{align*}
where we used $f(\bx,\bx)=\bw^\top(\bphi(\bx)-\bphi(\bx))=0$
from the second to third lines.
If the squared loss function $\ell(m)=(1-m)^2$ is used, $\ell(0)=1$
(cf. the implementation with the squared loss in Section~\ref{sec:impl}).
Therefore, the proposed PU-AUC risk estimator is unbiased.

\section{Proof of Generalization Error Bounds}
\label{proof:gen-err}
Here, we give the proofs of generalization error bounds 
in Section \ref{sec:theory-gen-err}.
The proofs are based on \citet{NIPS:Usunier+etal:2006}.

Let $\{\bx_i\}^m_{i=1}$ and $\{\bx'_j\}^n_{j=1}$
be two sets of samples drawn from the distribution
equipped with densities $q(\bx)$ and $q'(\bx)$,
respectively. 
Recall $\cF$ be a function class of bounded hyperplanes:
\begin{align*}
\cF:=\{ f(\bx,\bx')=\langle\bw, \bphi(\bx)-\bphi(\bx')\rangle
	\mid \|\bw\|\leq C_{\bw}; ~~ 
	\forall \bx\colon \|\bphi(\bx)\|\leq C_{\bphi} \},
\end{align*}
where $C_{\bw}>0$ and $C_{\bphi}>0$ are certain positive constants.
Then, the AUC risk over distributions $q$ and $q'$ 
and its empirical version can be expressed as
\begin{align*}
R(f)&:=\Eq[\Eqp[\ell(f(\bx,\bx'))]] , \\
\Rh(f)&:=\frac{1}{mn}\sum^m_{i=1}\sum^n_{j=1} 
	\ell(f(\bx_i,\bx'_j)) .
\end{align*}
For convenience, we define 
\begin{align*}
h(\delta):=2\sqrt{2}LC_\ell C_{\bw}C_{\bphi}
	+\frac{3}{2}\sqrt{2\log(2/\delta)} .
\end{align*} 

We first have the following theorem:
\begin{theorem}
\label{thm:mn-risk-bounds}
For any $\delta>0$, the following inequality holds
with probability at least $1-\delta$ for any $f\in\cF$:
\begin{align*}
R(f)-\Rh(f) &\leq 
	h(\delta)\frac{1}{\sqrt{\min(n,n')}} .
\end{align*}
\end{theorem}
\begin{proof}
By slightly modifying Theorem $7$ in \citet{NIPS:Usunier+etal:2006}
to fit our setting, 
for any $\delta>0$, with probability at least $1-\delta$
for any $f\in\cF$, we have
\begin{align}
R(f)-\Rh(f) &\leq \frac{2LC_\ell C_{\bw}\sqrt{\max(n,n')}}{nn'}
	\sqrt{\sum^n_{i=1}\sum^{n'}_{j=1}
		\|\bphi(\bx_i)-\bphi(\bx_j)\|^2} \notag \\
	&\phantom{\leq}
		+ 3\sqrt{\frac{\log(2/\delta)}{2\min(n,n')}} .
	\label{eq:modified-thm7}
\end{align}
Applying the inequality
\begin{align*}
\sum^n_{i=1}\sum^{n'}_{j=1}
		\|\bphi(\bx_i)-\bphi(\bx_j)\|^2 
&\leq n'\sum^n_{i=1}\|\bphi(\bx_i)\|^2
	+n\sum^{n'}_{j=1}\|\bphi(\bx_j)\|^2 \\
&\leq 2nn' C_{\bphi}^2 , 	
\end{align*}
to the first term in Eq.~\eqref{eq:modified-thm7}, we obtain the theorem.
\end{proof}

By using Theorem~\ref{thm:mn-risk-bounds},
we prove the risk bounds of the PU-AUC and NU-AUC risks:
\begin{lemma}
For any $\delta>0$, the following inequalities hold
separately with probability at least $1-\delta$ for any $f\in\cF$:
\begin{align*}
R_\PU(f)-\Rh_\PU(f) &\leq
	h(\delta/2)  
	\Big(\frac{1}{\thetan\sqrt{\min(\np,\nun)}}
	+\frac{\thetap}{\thetan\sqrt{\np}}
	\Big) , \\
R_\NU(f)-\Rh_\NU(f) &\leq
	h(\delta/2)  
	\Big(\frac{1}{\thetap\sqrt{\min(\nn,\nun)}}
	+\frac{\thetan}{\thetap\sqrt{\nn}}
	\Big) .	
\end{align*}
\end{lemma}
\begin{proof}
Recall that the PU-AUC and NU-AUC risks are expressed as
\begin{align*}
R_\PU(f)&=\frac{1}{\thetan}\Ep[\Eu[\ell(f(\bxp,\bxu))]]
	-\frac{\thetap}{\thetan}\Ep[\Epb[\ell(f(\bxp,\bbarxp))]] , \\
R_\NU(f)&=\frac{1}{\thetap}\Eu[\En[\ell(f(\bxu,\bxn))]]
	-\frac{\thetan}{\thetap}\En[\Enb[\ell(f(\bxn,\bbarxn))]] .	
\end{align*}

Based on Theorem~\ref{thm:mn-risk-bounds},
for any $\delta>0$, we have these uniform deviation bounds 
with probability at least $1-\delta/2$:
\begin{align*}
\sup_{f\in\cF}\Big(\Ep[\Eu[\ell(f(\bxp,\bxu))]]
	-\frac{1}{\np\nun}
	\sum^\np_{i=1}\sum^\nun_{k=1}\ell(f(\bxp_i,\bxu_k))\Big)
	&\leq h(\delta/2)\frac{1}{\sqrt{\min(\np,\nun)}} , \\
\sup_{f\in\cF}\Big(\Eu[\En[\ell(f(\bxu,\bxn))]]
	-\frac{1}{\nn\nun}
	\sum^\nun_{k=1}\sum^\nn_{j=1}\ell(f(\bxu_k,\bxn_j))\Big)
	&\leq h(\delta/2)\frac{1}{\sqrt{\min(\nn,\nun)}} , \\
\sup_{f\in\cF}\Big(\Ep[\Epb[\ell(f(\bxp,\bbarxp))]]
	-\frac{1}{\np^2}\sum^\np_{i=1}\sum^\np_{i'=1}
	\ell(f(\bxp_i,\bbarxp_{i'}))\Big) 
	&\leq h(\delta/2)\frac{1}{\sqrt{\np}} ,	 \\
\sup_{f\in\cF}\Big(\En[\Enb[\ell(f(\bxn,\bbarxn))]]
	-\frac{1}{\nn^2}\sum^\nn_{j=1}\sum^\nn_{j'=1}
	\ell(f(\bxn_j,\bbarxn_{j'}))\Big) 	
	&\leq h(\delta/2)\frac{1}{\sqrt{\nn}} .			
\end{align*}

Simple calculation showed that for any $\delta>0$, 
with probability $1-\delta$, we have
\begin{align}
\sup_{f\in\cF}\Big(
	R_\PU(f) - \Rh_\PU(f)\Big) 
	&\leq
	\frac{1}{\thetan}
	\sup_{f\in\cF}\Big(\Ep[\Eu[\ell(f(\bxp,\bxu))]]
	-\frac{1}{\np\nun}
	\sum^\np_{i=1}\sum^\nun_{k=1}\ell(f(\bxp_i,\bxu_k))\Big) \notag \\
	&\phantom{=}+\frac{\thetap}{\thetan}\sup_{f\in\cF}\Big(
	\Ep[\Epb[\ell(f(\bxp,\bbarxp))]]
	-\frac{1}{\np^2}\sum^\np_{i=1}\sum^\np_{i'=1}
	\ell(f(\bxp_i,\bbarxp_{i'}))\Big) \notag \\		
	&\leq h(\delta/2)
	\Big(\frac{1}{\thetan\sqrt{\min(\np,\nun)}}
	+\frac{\thetap}{\thetan\sqrt{\np}}
	\Big) ,	
	\label{eq:unif-pu-risk}
\end{align}
where we used
\begin{align*}
\sup(x+y)&\leq \sup(x)+\sup(y) , \\
R_\PU(f)&\leq
	\frac{1}{\thetan}\Ep[\Eu[\ell(f(\bxp,\bxu))]]
	+\frac{\thetap}{\thetan}\Ep[\Epb[\ell(f(\bxp,\bbarxp))]] .  
\end{align*}
Similarly, for the NU-AUC risk, we have
\begin{align}
\sup_{f\in\cF}\Big(
	R_\NU(f) - \Rh_\NU(f)\Big) 
	&\leq h(\delta/2)
	\Big(\frac{1}{\thetap\sqrt{\min(\nn,\nun)}}
	+\frac{\thetan}{\thetap\sqrt{\nn}}
	\Big) .	
	\label{eq:unif-nu-risk}
\end{align}
Eqs.~\eqref{eq:unif-pu-risk} and \eqref{eq:unif-nu-risk}
conclude the lemma.
\end{proof}
Finally, we give the proof of Theorem~\ref{thm:pu-nu-gen-err}.
\begin{proof}
Assume the loss satisfying $\ellzo(m)\leq M\ell(m)$.
We have $I(f)\leq M R(f)$.
When $M=1$ such as $\ells(m)$ and $\elle(m)$, 
$I(f)\leq R(f)$ holds.
This observation yields Theorem~\ref{thm:pu-nu-gen-err}.
\end{proof}

Next, we prove the generalization error bounds of 
the PNPU-AUC and PNNU-AUC risks in Theorem~\ref{thm:punu-pnu-gen-err}.
We first prove the following risk bounds:
\begin{lemma}
For any $\delta>0$, the following inequalities hold
separately with probability at least $1-\delta$ for all $f\in\cF$:
\begin{align*}
R_\PNPU^\gamma(f)-\Rh_\PNPU^\gamma(f) 
	&\leq h(\delta/3)
	\Big(	
	\frac{1-\gamma}{\sqrt{\min(\np,\nn)}} 	
	+\frac{\gamma }{\thetan\sqrt{\min(\np,\nun)}}
	+\frac{\thetap\gamma}{\thetan\sqrt{\np}}	
	\Big) , \\
R_\PNNU^\gamma(f)-\Rh_\PNNU^\gamma(f) 
	&\leq h(\delta/3)
	\Big(	
	\frac{1-\gamma}{\sqrt{\min(\np,\nn)}} 	
	+\frac{\gamma }{\thetap\sqrt{\min(\nn,\nun)}}
	+\frac{\thetan\gamma}{\thetap\sqrt{\nn}}	
	\Big) .	
\end{align*}
\end{lemma}
\begin{proof}
Recall the PNPU-AUC and PNNU-AUC risks:
\begin{align*}
R_\PNPU^\gamma(f)
	&:=(1-\gamma)R_\PN(f)+\gamma R_\PU(f), \\
R_\PNNU^\gamma(f)
	&:=(1-\gamma)R_\PN(f)+\gamma R_\NU(f) .
\end{align*}

Based on Theorem~\ref{thm:mn-risk-bounds},
for any $\delta>0$, we have these uniform deviation bounds 
with probability at least $1-\delta/3$:
\begin{align*}
\sup_{f\in\cF}\Big(\Ep[\En[\ell(f(\bxp,\bxn))]]
	-\frac{1}{\np\nn}
	\sum^\np_{i=1}\sum^\nn_{j=1}\ell(f(\bxp_i,\bxn_j))\Big)
	&\leq h(\delta/3)\frac{1}{\sqrt{\min(\np,\nn)}} , \\
\sup_{f\in\cF}\Big(\Ep[\Eu[\ell(f(\bxp,\bxu))]]
	-\frac{1}{\np\nun}
	\sum^\np_{i=1}\sum^\nun_{k=1}\ell(f(\bxp_i,\bxu_k))\Big)
	&\leq h(\delta/3)\frac{1}{\sqrt{\min(\np,\nun)}} , \\
\sup_{f\in\cF}\Big(\Eu[\En[\ell(f(\bxu,\bxn))]]
	-\frac{1}{\nn\nun}
	\sum^\nun_{k=1}\sum^\nn_{j=1}\ell(f(\bxu_k,\bxn_j))\Big)
	&\leq h(\delta/3)\frac{1}{\sqrt{\min(\nn,\nun)}} , \\
\sup_{f\in\cF}\Big(\Ep[\Epb[\ell(f(\bxp,\bbarxp))]]
	-\frac{1}{\np^2}\sum^\np_{i=1}\sum^\np_{i'=1}
	\ell(f(\bxp_i,\bbarxp_{i'}))\Big) 
	&\leq h(\delta/3)\frac{1}{\sqrt{\np}} ,	 \\
\sup_{f\in\cF}\Big(\En[\Enb[\ell(f(\bxn,\bbarxn))]]
	-\frac{1}{\nn^2}
	\sum^\nn_{j=1}\sum^\nn_{j'=1}\ell(f(\bxn_j,\bbarxn_{j'}))\Big) 			
	&\leq h(\delta/3)\frac{1}{\sqrt{\nn}} .			
\end{align*}

Combining three bounds from the above, for any $\delta>0$, 
with probability $1-\delta$, we have
\begin{align*}
&\sup_{f\in\cF}\Big(R_\PNPU^\gamma(f) - \Rh_\PNPU^\gamma(f)\Big) \\
	&\leq (1-\gamma) 
	\sup_{f\in\cF}\Big(\Ep[\En[\ell(f(\bxp,\bxn))]]
	-\frac{1}{\np\nn}
	\sum^\np_{i=1}\sum^\nn_{j=1}\ell(f(\bxp_i,\bxn_j))\Big) \\
	&\phantom{\leq}+\frac{\gamma}{\thetan}
	\sup_{f\in\cF}\Big(\Ep[\Eu[\ell(f(\bxp,\bxu))]]-\frac{1}{\np\nun}
	\sum^\np_{i=1}\sum^\nun_{k=1}\ell(f(\bxp_i,\bxu_k))\Big) \\
	&\phantom{\leq}+\frac{\gamma\thetap}{\thetan}
	\sup_{f\in\cF}\Big(\Ep[\Epb[\ell(f(\bxp,\bbarxp))]]
	-\frac{1}{\np^2}\sum^\np_{i=1}\sum^\np_{i'=1}
	\ell(f(\bxp_i,\bbarxp_{i'}))	\Big) \\
	&\leq h(\delta/3) 	
	\Big(	
	\frac{1-\gamma}{\sqrt{\min(\np,\nn)}} 	
	+\frac{\gamma }{\thetan\sqrt{\min(\np,\nun)}}
	+\frac{\gamma\thetap}{\thetan\sqrt{\np}}	
	\Big) .	
\end{align*}
This concludes the risk bounds of the PNPU-AUC risk.

Similarly, we prove the risk bounds of the PNNU-AUC risk.
\end{proof}
Again, $I(f)\leq R(f)$ holds in our setting.
This leads to Theorem~\ref{thm:punu-pnu-gen-err}.

\section{Proof of Variance Reduction}
\label{proof:var}
Here, we give the proof of Theorem~\ref{thm:var-red-pnpu-pnnu}.
\begin{proof}
The empirical PNPU-AUC risk can be expressed as
\begin{align*}
\Rh_\PNPU^\gamma(f)&=(1-\gamma)\Rh_\PN(f)+\gamma \Rh_\PU(f) \\
&=\frac{1-\gamma}{\np\nn}\sum^\np_{i=1}\sum^\nn_{j=1}
	\ell(f(\bxp_i,\bxn_j)) 
	+\frac{\gamma}{\thetan\np\nun}
	\sum^\np_{i=1}\sum^\nun_{k=1}\ell(f(\bxp_i,\bxu_k))]] \\
	&\phantom{=}-\frac{\gamma\thetap}{\thetan\np^2}
	\sum^\np_{i=1}\sum^\np_{i'=1}\ell(f(\bxp_i,\bbarxp_{i'}))]] .
\end{align*} 
Assume $\nun\to\infty$, we obtain  
\begin{align*}
\Var[\Rh_\PNPU^\gamma(f)]&=
	\frac{(1-\gamma)^2}{\np\nn}\sigmapn^2(f)
	+\frac{\gamma^2\thetap^2}{\np^2\thetan^2}\sigmapp^2(f)
	+\frac{(1-\gamma)\gamma}{\thetan\np}\taupnpu(f) \\
	&\phantom{=}-
	\frac{\gamma^2\thetap}{\thetan^2\np}\taupupp(f)
	-\frac{(1-\gamma)\gamma\thetap}{\thetan\np}
	\taupnpp(f) \\ 
&=(1-\gamma)^2 \psipn + \gamma^2\psipu + (1-\gamma)\gamma\psipp ,
\end{align*}
where the terms divided by $\nun$ are disappeared.
Setting the derivative with respect to $\gamma$ at zero,
we obtain the minimizer in Eq.~\eqref{eq:gam-pnpu}.

For the empirical PNNU-AUC risk, when $\nun\to\infty$, 
we obtain
\begin{align*}
\Var[\Rh_\PNNU(g)]&=
	\frac{(1-\gamma)^2}{\np\nn}\sigmapn^2(g)
	+\frac{\gamma^2\thetan^2}{\nn^2\thetap^2}\sigmann^2(g)
	+\frac{(1-\gamma)\gamma}{\thetap\nn}\taupnnu(g) \\
	&\phantom{=}-
	\frac{\gamma^2\thetan}{\thetap^2\nn}\taununn(g) 
	-\frac{(1-\gamma)\gamma\thetan}{\thetap\nn}
	\taupnnn(g) \\
&=(1-\gamma)^2 \psipn + \gamma^2 \psinu + (1-\gamma)\gamma \psinn .
\end{align*}
Setting the derivative with respect to $\gamma$ at zero,
we obtain the minimizer in Eq.~\eqref{eq:gam-pnnu}.
\end{proof}

\section{Statistics of Data Sets}
\label{app:stats_data}
Table~\ref{tab:spec_data} summarizes the statistics of the data sets 
used in our experiments.
The class balance is the number of positive samples
divided by that of total samples.
The sources of data sets are as follows:    
the IDA Benchmark Repository (IDA) \citep{IDA:MLJ:Ratsch+etal:2001},
the UCI Machine Learning Repository (UCI) \citep{UCI:Lichman:2013},
the LIBSVM data sets (LIBSVM) \citep{LibSVM:Chang+etal:2011},
the Semi-Supervised Learning Book (SSL) \citep{book:Chapelle+etal:2006},
and the Amazon Review (Amazon$7$) \citep{MLJ:Blondel+etal:2013}.

\begin{table}[!ht]
	\centering
	\caption{The statistics of the data sets.
	The source of data sets is as follows:    
	the IDA Benchmark Repository (IDA) \citep{IDA:MLJ:Ratsch+etal:2001},
	the UCI Machine Learning Repository (UCI) \citep{UCI:Lichman:2013},
	the LIBSVM data sets (LIBSVM) \citep{LibSVM:Chang+etal:2011},
	the Semi-Supervised Learning Book (SSL) \citep{book:Chapelle+etal:2006},
	and the Amazon Review (Amazon$7$) \citep{MLJ:Blondel+etal:2013}.}
	\label{tab:spec_data}
	\begin{tabular}{lrrrl}
	\toprule
	Data set 
	& \multicolumn{1}{l}{Dimension} 
	& \multicolumn{1}{l}{\#samples} 
	& \multicolumn{1}{l}{Class balance} 
	& Source 
	\\
	\midrule
	Banana & $2$ & $5,300$ & $0.45$ & IDA \\ 
	skin\_nonskin & $3$ & $245,057$ & $0.21$ & LIBSVM \\ 
	cod-rna & $8$ & $331152$ & $0.67$ & LIBSVM \\ 
	Magic & $10$ & $19020$ & $0.35$ & UCI \\ 
	Image & $18$ & $2,310$ & $0.39$ & IDA \\ 
	SUSY & $18$ & $5,000,000$ & $0.46$ & LIBSVM \\ 
	Ringnorm & $20$ & $7,400$ & $0.50$ & IDA \\ 
	Twonorm 	& $20$ & $7,400$ & $0.50$ & IDA \\ 
	Waveform & $21$ & $5,000$ & $0.33$ & IDA \\ 
	covtype & $54$ & $581,012$ & $0.51$ & LIBSVM \\ 
	phishing & $68$ & $11,055$ & $0.44$ & LIBSVM \\ 
	a$9$a & $83$ & $48,842$ & $0.24$ & LIBSVM \\ 
	mushrooms & $112$ & $8,124$ & $0.48$ & LIBSVM \\ 
	USPS & $241$ & $1,500$ & $0.20$ & SSL \\ 
	w$8$a & $300$ & $64,700$ & $0.03$ & LIBSVM \\ 
	rcv$1$ & $47,236$ & $697,641$ & $0.53$ & LIBSVM \\ 
	amazon$2$ & $262,144$ & $1,149,374$ & $0.18$ & Amazon$7$ \\ 
	news$20$ & $1,355,191$ & $19,996$ & $0.50$ & LIBSVM \\ 
	\bottomrule 
	\end{tabular}
\end{table}

\newpage
\bibliographystyle{plainnat-reversed}
\bibliography{pnu_auc}

\end{document}